\documentclass{article}

\usepackage{microtype}
\usepackage{graphicx}
\usepackage{subfigure}
\usepackage{booktabs} 
\usepackage{enumitem}
\usepackage{hyperref}
\usepackage{wrapfig}

\usepackage[accepted]{icml2024}

\usepackage{amsmath}
\usepackage{amssymb}
\usepackage{mathtools}
\usepackage{amsthm}

\usepackage[capitalize,noabbrev]{cleveref}

\theoremstyle{plain}
\newtheorem{theorem}{Theorem}[section]

\newtheorem{lemma}[theorem]{Lemma}

\theoremstyle{definition}
\newtheorem{definition}[theorem]{Definition}

\theoremstyle{remark}

\usepackage[textsize=tiny]{todonotes}

\icmltitlerunning{Provably Better Explanations with Optimized Aggregation of Feature Attributions}

\begin{document}

\twocolumn[
\icmltitle{Provably Better Explanations with Optimized Aggregation of \\ Feature Attributions}



\icmlsetsymbol{equal}{*}

\begin{icmlauthorlist}
\icmlauthor{Thomas Decker}{lmu,siemens}
\icmlauthor{Ananta R. Bhattarai}{siemens,tum}
\icmlauthor{Jindong Gu}{oxford}
\icmlauthor{Volker Tresp}{lmu,mcml}
\icmlauthor{Florian Buettner}{siemens,goethe,dkfz}

\end{icmlauthorlist}

\icmlaffiliation{lmu}{LMU Munich}
\icmlaffiliation{siemens}{Siemens AG}
\icmlaffiliation{tum}{Technical University of Munich}
\icmlaffiliation{mcml}{Munich Center for Machine Learning (MCML)}
\icmlaffiliation{dkfz}{German Cancer Research Center (DKFZ)}
\icmlaffiliation{oxford}{University of Oxford}

\icmlaffiliation{goethe}{Goethe University Frankfurt}
\icmlcorrespondingauthor{Thomas Decker}{thomas.decker@siemens.com}
\icmlcorrespondingauthor{Florian Buettner}{florian.buettner@dkfz.de}

\icmlkeywords{Machine Learning, ICML}

\vskip 0.3in
]



\printAffiliationsAndNotice{}  

\begin{abstract}
Using feature attributions for post-hoc explanations is a common practice to understand and verify the predictions of opaque machine learning models. Despite the numerous techniques available, individual methods often produce inconsistent and unstable results, putting their overall reliability into question. In this work, we aim to systematically improve the quality of feature attributions by combining multiple explanations across distinct methods or their variations. For this purpose, we propose a novel approach to derive optimal convex combinations of feature attributions that yield provable improvements of desired quality criteria such as robustness or faithfulness to the model behavior. Through extensive experiments involving various model architectures and popular feature attribution techniques, we demonstrate that our combination strategy consistently outperforms individual methods and existing baselines.
\end{abstract}
\begin{figure*}[t!]
    \centering
    \includegraphics[trim={0 0.5cm 0 0}, width=\textwidth]{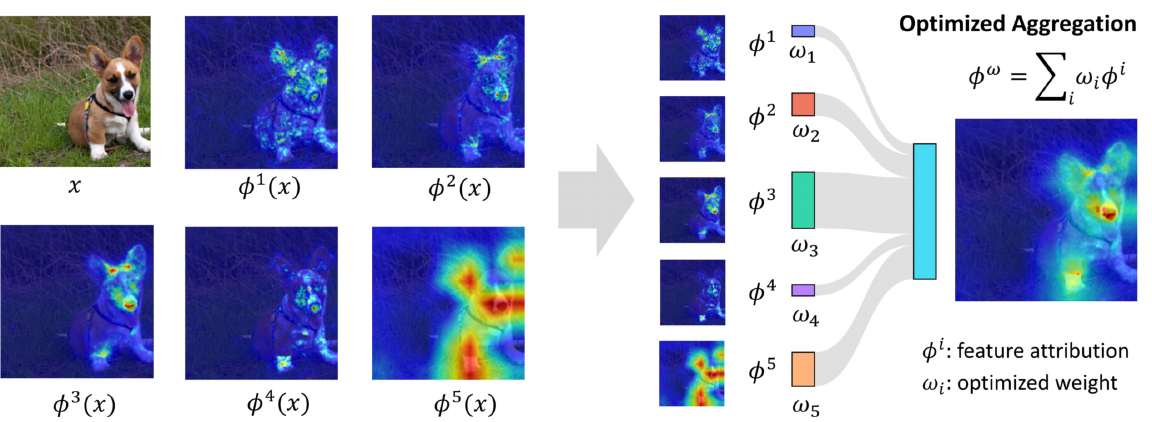}
    \caption{\textbf{Disagreement across attribution methods (left):}  Different feature attribution methods $(\phi^1, \dots , \phi^5)$ provide distinct perspectives about which particular features of an input $x$ are important for an opaque model prediction $f(x)$. Oftentimes they tend to disagree causing ambiguity about which inputs truly matter. \textbf{Our Optimized Aggregation approach (right):}  We study how to combine all individual attribution results fruitfully to attain better explanations. We propose a novel aggregation approach to retrieve optimal convex weights $\omega_i$ such that the aggregated feature attribution $\phi^{\omega} = \sum_i \omega_i \phi^i$ is provably more robust and more faithful to the underlying model.}
    \label{method:fig}
\end{figure*}
\section{Introduction}
The practice of quantifying the influence of individual features through attribution methods has been established as a popular paradigm to enhance the transparency of complex machine learning models. These approaches typically produce heatmaps highlighting individual input features, such as pixels or image regions, relevant to a specific model prediction (see Figure \ref{method:fig}, left). However, while a multitude of techniques has been developed for this purpose, concerns and doubts regarding the reliability of individual methods persist \cite{adebayo2018sanity, adebayo2020debugging, zhou2022feature}.
For instance, empirical evidence has revealed that single methods often exhibit unreasonable sensitivity to input perturbations \cite{kindermans2019reliability, alvarez2018robustness, dombrowski2019explanations,ghorbani2019interpretation, lin2023robustness} and critically depend on the concrete hyperparameter choice \cite{bansal2020sam, sturmfels2020visualizing, Pahde_2023_CVPR}. This lack of explanation robustness not only causes unstable attribution results but can even be exploited for malicious manipulations \cite{baniecki2023adversarial}. On top of that, some methods may fail to identify relevant features \cite{hooker2019benchmark, zhou2022feature}
and different techniques frequently disagree substantially when explaining the same prediction \cite{krishna2022disagreement, neely2021order}. These findings contest the actual fidelity of a single attribution result for the purpose model explainability. \\
In addition to these observations, there is a growing body of theoretical work that highlights the limitations of individual attribution methods \cite{nie2018theoretical,sixt2020explanations,kumar2021shapley,bilodeau2022impossibility, fokkema2022attribution}. More specifically, in \cite{han2022explanation}, the authors establish a "no-free lunch" theorem for model explanations, which implies that a single attribution method cannot universally approximate the behavior of any model faithfully.\\
Nevertheless, each feature attribution method derives importance based on different mechanisms and each can be associated with individual benefits and shortcomings. As a consequence, the question arises of how to best combine them to attain better explainability of opaque predictions.\\
In this work, we explore the capabilities of convex combinations across different attribution results to improve the overall reliability of explanations. Guided by established quality criteria for feature attributions \cite{nauta2023anecdotal}, we propose an effective strategy to derive convex weightings such that the corresponding aggregation of different outcomes yields significant improvements in robustness and faithfulness. This is underpinned by a theoretical analysis showing that the improvements in relevant quality metrics are provable and even close to optimal with high probability. 
\\
Our specific contributions are the following:
\begin{itemize}[leftmargin=12pt]
    \item We introduce an innovative approach for combining the results generated by various feature attribution methods or different variants of the same method.
    \item We show that our method can be effectively employed to optimize explanations according to commonly used measures of quality, including robustness and faithfulness to the model's behavior based on a unifying framework.
    \item We conduct a rigorous theoretical analysis establishing provable improvements of explanation quality and
    corresponding optimality bounds for our approach.
    \item We manifest these findings through a series of experiments involving popular feature attribution techniques and model architectures, consistently outperforming existing baselines and individual methods.
\end{itemize}

\section{Problem Setup}
Our goal is to enhance the reliability of explanations by developing effective strategies for combining diverse feature attribution results. To illustrate, let's consider explaining a prediction, denoted as $f(x)$, of a classification model $f:\mathbb{R}^{d} \rightarrow \mathbb{R}$, for an input instance $x \in \mathbb{R}^{d}$. A feature attribution method $\phi:\mathbb{R}^{d} \rightarrow [0,1]^{d}$, explains the prediction $f(x)$ by associating to each separate input $x_{i}$ a normalized importance score $\phi_{i}(x)$. Suppose we have access to $k$ different attribution methods, denoted by $\phi^1, \dots, \phi^k$, each offering a distinct perspective. Further, let $\omega_1, \dots, \omega_k$ be scalar weights such that $\sum_i \omega_i = 1$, and each weight is non-negative $\left(\omega_i \ge 0\right)$. To aggregate distinct attribution outcomes we consider the weighted sum $\sum_i \omega_i \phi^i(x)$ yielding a novel explanation that combines individual insights. Our objective is to determine prediction-specific weights $\omega_i$ in a manner that provably improves desired quality metrics. This ultimately leads to more reliable and better explanations via aggregation (see Figure \ref{method:fig}, right).
\section{Background and Related Work}
\subsection{Measuring attribution quality}
Evaluating the fidelity of explanations is a challenging endeavor due to missing knowledge about an objective ground truth. However, several quantitative metrics have been proposed to assess different aspects concerning the quality of feature attribution results \cite{nauta2023anecdotal, hedstrom2023quantus}. In our study, we focus on two prominent categories of explanation quality:
\paragraph{Robustness}
Many explanation methods exhibit instabilities under small input perturbations \cite{alvarez2018robustness} leading to significantly different feature attribution results for almost identical inputs. While this not only casts doubts regarding the explanatory integrity of the considered technique, it might further be exploited to manipulate explanations intentionally \cite{baniecki2023adversarial}. A popular metric to quantitatively measure attribution robustness is Max-Sensitivity \cite{yeh2019fidelity}:
 \begin{align*}
     \text{SENS}_{\textit{MAX}} : \max_{\lVert \varepsilon \rVert \le \delta} \quad \lVert \phi(x) - \phi(x + \varepsilon) \rVert
 \end{align*}
 This quantity is typically estimated using a Monte Carlo approach by sampling a fixed number of small perturbations, evaluating the explanations, and storing the maximal distortion. An alternative metric for attribution robustness is Average-Sensitivity \cite{bhatt2021evaluating}:
  \begin{align*}
     \text{SENS}_{\textit{AVG}} : \mathbb{E}_{\varepsilon} \left[ \lVert \phi(x) - \phi(x+\varepsilon) \rVert \right]
 \end{align*}
where $\varepsilon \sim \mathbb{P}_{\varepsilon}$ is a small random input perturbation, typically either Gaussian or uniformly distributed with mean zero. 
\paragraph{Faithfulness}
The goal of faithfulness metrics is to measure how aligned an attribution result is with the actual model behavior in the sense that perturbing important features should also alter the model prediction accordingly. While different mathematical formulations have been proposed, a prominent choice is Infidelity \cite{yeh2019fidelity}:
\begin{align*}
    \text{INFD}: \mathbb{E}_{I} \left[ ( I^T\phi(x) - (f(x)-f(x-I)))^2 \right]
\end{align*} 
Here, $I \in \mathbb{R}^d$ describes a probabilistic perturbation such as replacing random parts of $x$ with a fixed baseline value or Gaussian noise \cite{yeh2019fidelity}. Similarly, \cite{bhatt2021evaluating} proposed to quantify faithfulness to the model's behavior using a correlation measure:
\begin{align*}
    \text{FCOR}: \text{corr}_I \left(I^T\phi(x),  (f(x)-f(x-I)) \right)
\end{align*}
Thus, Faithfulness Correlation (FCOR) measures how correlated the attribution scores are with prediction changes under corresponding input modifications. 
\paragraph{Other metrics}
Beyond robustness and faithfulness, additional dimensions of explanation quality have also been investigated in the literature. Alignment metrics \cite{ARRAS202214,decker2023does} measure to which extent an explanation matches a desirable ground truth derived from domain knowledge and randomization-based sanity checks \cite{adebayo2018sanity, hedstrom2023sanity} ensure a sufficient dependence of the attribution result on the examined model. Moreover, Complexity metrics \cite{bhatt2021evaluating, chalasani2020concise} quantify how comprehensible a model explanation is given the premise that sparser attributions are more informative to humans due to reduced cognitive load. Please refer to \cite{nauta2023anecdotal} for a more comprehensive overview of available metrics.

\subsection{Aggregating explanations}
The idea of aggregating multiple feature attribution results \textbf{within} the same method is already anchored in popular explainability techniques. SmoothGrad \cite{smilkov2017smoothgrad} and UniformGrad \cite{wang2020smoothed} combine gradients in the proximity of the input and VarGrad \cite{adebayo2018sanity} uses the variance of gradients within a neighborhood to derive feature importance. Similarly, Integrated Gradients \cite{sundararajan2017axiomatic} and GradSHAP \cite{erion2021improving} aggregate gradients along a specific path towards predetermined baseline values. While such techniques combine gradients following input perturbation, NoiseGrad \cite{bykov2022noisegrad} averages gradients under model parameter modifications to form a final explanation. In \cite{rebuffi2020there} the authors analyze how the combination of attribution results obtained from different layers can improve the final explanation. The authors of \cite{bhatt2021evaluating} propose to enhance explanations by combining the Shapley Values of an instance with the ones obtained from its nearest neighbors in the training dataset. \\
On the other hand, the idea of aggregating attribution results \textbf{across} distinct methods has received considerably less attention. In \cite{rieger2019aggregating}, the authors propose two basic ways to combine distinct explanations which are defined as follows. $\text{AGG}_{\textit{Mean}}=\frac{1}{k}\sum_{i=1}^k \phi^i(x)$ simply averages different attribution outcomes and $\text{AGG}_{\textit{Var}}$ incorporates also feature-wise variability to downgrade the importance of features where methods tend to disagree on:
\begin{align*}
    \text{AGG}_{\textit{Var}} = \frac{1}{k}\sum_{i=1}^k \frac{\phi^i(x)}{\sigma(\phi^{1}, \dots, \phi^{k}) + \epsilon }
\end{align*} where $\sigma(\phi^{1}, \dots, \phi^{k}) \in \mathbb{R}^d$ describes the feature-wise standard deviation across the different attribution results and $\epsilon$ is a small constant promoting numerical stability. \\
Nevertheless, a theoretically grounded strategy of how to best combine different attribution results for desired improvements is still missing and we aim to address this gap in the remainder of this paper. 
 
\section{Optimizing Explanations with Aggregation}
\paragraph{Generalized L2 metrics for explanations}
In this section, we introduce a general class of quality metrics for explanation methods that can efficiently be improved via cross-method combination as shown later. 
\begin{definition}
    Let $\mathcal{Q}: \mathbb{R}^d\rightarrow \mathbb{R}$ a be a quality metric for feature attribution results. Then, $\mathcal{Q}$ belongs to the class of generalized $L2$ metrics if there exist suitable random variables $\gamma_1 \in \mathbb{R}^{g \times d}$ and $\gamma_2 \in \mathbb{R}^{g}$ such that :
    \begin{align*}
    \mathcal{Q}(\phi(x))=\mathbb{E}_{\gamma_1, \gamma_2}[ \lVert \gamma_1 \phi(x) - \gamma_2\rVert_2^2 ]
    \end{align*}
\end{definition}
Conceptually, any such metric evaluates the quality of an attribution result $\phi(x)$ using the following intuitive principle. First, a linear query $\gamma_1$ is applied to extract certain information from the attribution results. Second, the obtained information content is compared to a desired query outcome $\gamma_2$ using the squared Euclidean distance. Thus, a smaller value of $\mathcal{Q}$ implies a better attribution result according to the considered criteria. Evaluating such metrics can simply be performed by estimating the expectation with a finite set of metric evaluation samples denoted by $\{(\gamma_1^{(j)}, \gamma_2^{(j)})\}_{i=1}^m$. \\ Note that common quality metrics introduced above are generalized $L2$ metrics. For example, Average-Sensitivity can be recovered in the following way: Let $\mathcal{I}_d$ be the $d$-dimensional identity matrix, then setting $\gamma_1 =\mathcal{I}_d$ and $\gamma_2 = \phi(x+\varepsilon)$ with $\varepsilon \sim \mathbb{P}_{\varepsilon}$ results in $\text{SENS}_{\textit{AVG}} $ with respect to the squared Euclidean norm. Similarly, Infidelity can be obtained by choosing $\gamma_1 = I^T$ and $\gamma_2 = f(x) -f(x-I)$. We further show in \autoref{app:B} how other categories of quality metrics can be expressed within this framework. \\
In conclusion, many established quality criteria for model explanations can be assessed based on a corresponding $L2$ formulation. Next, we show that enhancing such metrics through convex combinations leads to a well-posed optimization problem.
\paragraph{Deriving optimal weights}
Remember that our goal is to combine multiple attribution results to provably improve desired quality criteria. Suppose $k$ different attribution outcomes $\phi^1(x) \dots \phi^k(x)$ for which we seek optimal convex weight factors $\boldsymbol{\omega}=(\omega_1, \dots, \omega_k)$. First, note that evaluating an aggregated attribution result $\phi^{\omega}=\sum_i \omega_i \phi^i$ via a generalized $L2$ metric $\mathcal{Q}$ reads:
\begin{align*} \mathcal{Q}(\phi^{\omega}) = 
\mathbb{E}\left[ \lVert \gamma_1 \phi^{\omega} - \gamma_2\rVert_2^2 \right] = 
 \mathbb{E}\left[ \lVert (\gamma_1\Phi)\omega - \gamma_2\rVert_2^2 \right]
\end{align*}where $\Phi \in \mathbb{R}^{d\times k}$ describes the matrix of stacked individual attribution results $\Phi = (\phi^{1},\dots, \phi^{k})$. Therefore, optimizing for convex weights $\omega$ reduces to solving:
\begin{align*}
\min_{\omega} \;\; \mathbb{E}\left[ \lVert (\gamma_1\Phi)\omega - \gamma_2\rVert_2^2 \right] \quad \text{s.t.} \quad \omega_i \ge 0, \;\; \sum_{i=1}^k \omega_i = 1
\end{align*}
To ease the notation, we denote the set of feasible aggregation weights by $\Omega = \{ \omega \in \mathbb{R}^k: \omega_i \ge 0, \; \sum_{i=1}^k \omega_i = 1 \}$  and define $\gamma_{2, k:} \in \mathbb{R}^{g \times k}$ as matrix storing $k$ copies of $\gamma_2$ in its columns. By setting $\Gamma := (\gamma_1\Phi -\gamma_{2, k:})\in \mathbb{R}^{g \times k}$, it holds within the set of feasible  weights $\Omega$ that:
\begin{align*}
 \min_{\omega \in \Omega} \;\;  \mathcal{Q}(\phi^{\omega}) \;\;  \Leftrightarrow \;\; \min_{\omega \in \Omega} \;\;   \omega^T \mathbb{E}\left[ \Gamma^T\Gamma \right] \omega
\end{align*}
Hence, searching for the best way to aggregate different attribution outcomes ends up in a constrained quadratic program with convex constraints. This exhibits a global optimum and can efficiently be solved using corresponding numerical solvers \cite{boyd2004convex}. On top of that, this observation also enables us to optimize multiple generalized $L2$ metrics simultaneously as quadratic forms are additive.  Suppose we seek to improve $q$ independent metrics $\mathcal{Q}_1, \dots \mathcal{Q}_q$ with associated parameters $\Gamma_q$ as defined above. Then, for any scalers $\lambda_1, \dots \lambda_q$ we have:
\begin{align*}
    \sum_{j=1}^q \lambda_j \mathcal{Q}_j(\phi^{\omega})
 = \omega^T \mathbb{E}\left[\sum_{j=1}^q \lambda_j  \Gamma_j^T \Gamma_j \right] \omega 
\end{align*}
This implies that searching for convex weights that directly improve multiple metrics prioritized by $\lambda_i$ can also be expressed as a single constrained quadratic program and thus efficiently be solved.\\ In addition to its numerical appeal, optimizing explanations via aggregation in this way also comes with theoretical benefits in the form of provable improvement guarantees and probabilistic optimality bounds. 

\paragraph{Provable improvement through aggregation}
The following theorem allows us to precisely quantify the gain in explanation quality induced via convex aggregation. 
\begin{theorem}
    Let $\phi^{\omega}=\sum_i \omega_i \phi^i$ be the aggregated explanation, then the quality metric of $\phi^{\omega}$ is always at least as good as the weighted metrics of the individual attributions:
    \begin{align*}
    \mathcal{Q}(\phi^{\omega}) = \sum_i \omega_i \mathcal{Q}(\phi^i) - \mathbb{E}_{\gamma_1}\left[ \sum_i \omega_i \lVert \gamma_1 (\phi_i- \phi^{\omega})\rVert_2^2\right]
    \end{align*}
\end{theorem} Note that this result can be related to the error ambiguity decomposition for ensemble learning introduced in \cite{krogh1994neural} and we conduct the proof in \autoref{app:A}. The achievable gain via aggregation $\mathbb{E}_{\gamma_1}\left[ \sum_i \omega_i \lVert \gamma_1 (\phi_i- \phi^{\omega})\rVert_2^2\right] \ge 0$ depends on how diverse the different explanations behave under queries $\gamma_1$ compared to the aggregated one. Moreover, its non-negativity ensures that the quality of the aggregated explanation is at least as good as the equivalently weighted individual attribution qualities since lower values of $\mathcal{Q}$ imply improvements.

\begin{table*}[t!] \centering
\caption{$\text{SENS}_{\text{AVG}}$ ($\text{S}_{\text{AVG}}$) and $\text{SENS}_{\text{MAX}}$ ($\text{S}_{\text{MAX}}$) results for gradient-based attribution methods and different aggregation strategies across several model architectures. Our approach $\text{AGG}_{\textit{robust}}$ consistently outperforms all other techniques followed by $\text{AGG}_{\textit{opt}}$ as second best.}
         \begin{tabular}{c|*{2}{c}|*{2}{c}|*{2}{c}|*{2}{c}|*{2}{c}}
         \toprule
         Feature
         &  \multicolumn{2}{c}{VGG16} 
         &  \multicolumn{2}{c}{AlexNet}
         & \multicolumn{2}{c}{ResNet18} 
        & \multicolumn{2}{c}{MobileNetV2} 
        & \multicolumn{2}{c}{MLPMixer}
        \\ Attribution
       &  $\text{S}_{\textit{AVG}}\downarrow$ & $\text{S}_{\textit{MAX}}\downarrow$ & 
       $\text{S}_{\textit{AVG}}\downarrow$& $\text{S}_{\textit{MAX}}\downarrow$ & $\text{S}_{\textit{AVG}}\downarrow$& $\text{S}_{\textit{MAX}}\downarrow$ &      
       $\text{S}_{\textit{AVG}}\downarrow$& $\text{S}_{\textit{MAX}}\downarrow$& $\text{S}_{\textit{AVG}}\downarrow$& $\text{S}_{\textit{MAX}}\downarrow$ \\
         \midrule
         Saliency &0.994 &1.214&0.800& 0.942&0.964 & 1.163&1.022 &1.265 &1.200 &1.591 \\
         Guided BP &0.515 &0.650&0.430&0.512 &0.483 &0.611 &0.832 &1.060 &- &-\\
         DeepLift &0.893 & 1.090&0.791 &0.932 &0.857 &0.999 &0.894 &1.075 &0.961 &1.165\\
         IntGrad &0.888 &1.065&0.724& 0.854&0.838 &0.991 & 0.910 &1.084 &0.941 &1.141 \\
         InputxGrad &0.988 &1.214&0.807&0.957 &0.956 &1.157 & 1.029&1.292 &1.107 &1,419\\ 
         SmoothGrad  &0.784 &0.913&0.622&0.719 &0.779 &0.898 &0.858  &0.992 &0.643 &0.739\\
         VarGrad  &0.599 &0.949&0.571&0.914 &0.553 &0.829 &0.747 &1.183 &0.554 & 0.910\\
        \midrule
         $\text{AGG}_{\textit{Mean}}$  &0.596 &0.734&0.480&0.583 &0.529 & 0.644& 0.586  &0.724 &0.663 &0.853\\
         $\text{AGG}_{\textit{Var}}$  &0.582 &0.700&0.476&0.574 &0.518 &0.618 & 0.568 &\underline{0.686} & 0.631 & 0.788\\
         $\text{AGG}_{\textit{faith}}$ \textbf{(ours)}  &0.644 &0.833&0.535&0.679 &0.471 &0.578 & 0.792 &1.036 & 0.696 & 0.892\\
         $\text{AGG}_{\textit{opt}}$  \textbf{(ours)} &\underline{0.456} &\underline{0.584} &\underline{0.364} &\underline{0.449} &\underline{ 0.427}&\underline{0.538}&\underline{ 0.536} &0.701 &\underline{0.483} &\underline{0.642}\\
         $\text{AGG}_{\textit{robust}}$  \textbf{(ours)} &\textbf{0.424} &\textbf{0.543}&\textbf{0.349}&\textbf{0.426 }&\textbf{0.410} &\textbf{0.513} &\textbf{0.505} &\textbf{0.654} & \textbf{0.473}&\textbf{0.634}\\
         \bottomrule
     \end{tabular}
     \end{table*}
\paragraph{Generalization bounds for estimated weights }
Obtaining optimal weights usually requires approximating the objective based on a limited set of metric evaluation samples $\{(\gamma_1^{(j)}, \gamma_2^{(j)})\}_{i=1}^m$. Hence, the resulting estimate  $\hat{\omega} = \arg\min_{\omega \in \Omega} \frac{1}{m}\sum_{j=1}^m \lVert \gamma_1^{(j)}\phi^{\omega} - \gamma_2^{(j)} \rVert_2^2$ may deviate from the ideal combination weights as it might not generalize well to unseen metric evaluations. As a consequence, it would be desirable to ensure that the quality improvement with estimated aggregation weights is close enough to the best possible strategy concerning the entire quality metric $\mathcal{Q}$. The following theorem establishes a corresponding result.

\begin{theorem}
Consider a generalized $L2$ metric denoted by $\mathcal{Q}$ with $\max_{\gamma_1} \lVert \gamma_1 \rVert_1 \le c_1$ and $\max_{\gamma_1, \gamma_2} \lVert \gamma_1\phi^i  -\gamma_2 \rVert_2^2 \le c_2$. Additionally, let $\Omega$ represent the set of feasible weights $\omega$ and $\phi^{\omega}=\sum_{i=1}^k \omega_i \phi^i$ denote an aggregated feature attribution result. Suppose $\hat{\omega}$ is an estimate of aggregation weights obtained from $m$ metric evaluation samples given by:
\begin{align*}
    \hat{\omega} =& \arg\min_{\omega \in \Omega}\; \; \frac{1}{m}\sum_{j=1}^m \; \; \lVert \gamma_1^{(j)}\phi^{\omega} - \gamma_2^{(j)} \rVert_2^2 
\end{align*}
Then there exist a constant $C(c_1, c_2) >0 $ depending on $c_1$ and $c_2$ such that with probability of at least $(1-\delta)$:
\begin{align*}
\mathcal{Q}(\phi^{\hat{\omega}}) -    \min_{\omega \in \Omega} \mathcal{Q}(\phi^{\omega})  \le C \; \sqrt{\dfrac{4\log(16k/\delta)}{m}}
\end{align*}
\end{theorem} To prove this statement, we develop appropriate bounds on the Rademacher complexity of vector-valued functions based on a concentration result from \cite{maurer2016vector} and the specific properties of generalized $L2$ metrics over convex combinations of normalized feature attribution results. The full derivation and the precise expression for the constant $C$ are given in \autoref{app:A}. Intuitively, theorem 4.3 guarantees that the maximum potential deviation of our aggregation approach from the optimal improvement can be bounded with high probability. Moreover, the worst-case performance gap diminishes with order $\mathcal{O}(1/\sqrt{m})$ for increasing number of metric evaluation samples $m$.

\paragraph{Optimal aggregation for desired improvements} Based on the generalized framework above we propose different aggregation strategies to intentionally enhance specific properties of feature attribution results. To explicitly enhance explanation robustness we obtain combination weights $\omega^{\textit{robust}}$ by optimizing Average-Sensitivity as related $L2$ metric.
\begin{align*}
    \text{AGG}_{\textit{robust}}: \quad  
    \omega^{\textit{robust}} = \arg\min_{\omega \in \Omega}\;\text{SENS}_{\textit{AVG}}(\phi^{\omega})
\end{align*}
Equivalently, to optimize for faithfulness we can compute $\omega^{\textit{faith}}$ by considering Infidelity as underlying objective:
\begin{align*}
 \text{AGG}_{\textit{faith}}: \quad \omega^{\textit{faith}} = \arg\min_{\omega \in \Omega}\;\text{INFD}(\phi^{\omega})
\end{align*}
As a default strategy to increase explanation quality via aggregation, we further propose improving both metrics simultaneously. We coin this approach $\text{AGG}_{\textit{opt}}$ optimizing for better feature attributions more generically:
\begin{align*}
\text{AGG}_{\textit{opt}}: \quad \omega^{\textit{opt}} = \arg\min_{\omega \in \Omega}\;\text{INFD}(\phi^{\omega}) + \text{SENS}_{\textit{AVG}}(\phi^{\omega})
\end{align*}
\section{Experiments} We conducted a multifaceted empirical evaluation to investigate the capacities of our proposed aggregation strategies to intentionally enhance desired properties of explanations. All our aggregation strategies are optimized using only a small amount of metric evaluation samples to approximate the underlying metric ($m_{\textit{agg}}=50$). We explicitly test how well the improvements generalize to a larger sample of novel metric evaluations ($m_{\textit{eval}}=200$) and if they transfer to alternative quality measures. The findings presented in this section are based on the ImageNet \texttt{ILSVRC2012} dataset and concrete implementation details are documented in \autoref{app:C}. Accompanying source code is released at \texttt{https://github.com/thomdeck/aggopt}. 

\subsection{Quantitative evaluation of quality improvements}
\paragraph{Increasing robustness via $\text{AGG}_{\textit{robust}}$ and $\text{AGG}_{\textit{opt}}$}
We examine to which extent our aggregation approach can mitigate typically encountered instabilities of gradient-based explanations on convolutional models. For this purpose, we consider seven corresponding attribution techniques as well as four different ways of combining them including the two simple baselines $\text{AGG}_{\textit{Mean}}$ and $\text{AGG}_{\textit{Var}}$ and our proposed strategies $\text{AGG}_{\textit{robust}}$ and $\text{AGG}_{\textit{opt}}$. All resulting explanations are computed for 500 random samples from ImageNet across five popular computer vision models and we evaluated their robustness based on the metrics $\text{SENS}_{\text{AVG}}$ and $\text{SENS}_{\text{MAX}}$. The corresponding results in Table 1 indicate that our approach $\text{AGG}_{\textit{robust}}$ consistently outperforms all individual attribution methods as well as all other aggregations followed by $\text{AGG}_{\textit{opt}}$, which is almost always second best. Remember that $\text{AGG}_{\textit{robust}}$  directly optimizes for $\text{SENS}_{\textit{AVG}}$ using a small number of metric evaluations. Thus, the generalization performance for this metric, now evaluated with a higher number of unseen metric evaluation samples, is in line with our theoretical framework. On top of that, the additional superiority in terms of the alternative metric $\text{SENS}_{\textit{MAX}}$ demonstrates that our approach also improves attribution robustness in general.
\begin{table*}[t!] \centering
\caption{$\text{INFD}$ and $\text{FCOR}$ results for different attribution methods and aggregation strategies across several model architectures. Our approach $\text{AGG}_{\textit{faith}}$ consistently outperforms all other techniques and $\text{AGG}_{\textit{opt}}$ is either second best or comparable.}
\resizebox{\textwidth}{!}{
         \begin{tabular}{c|*{2}{c}|*{2}{c}|*{2}{c}|*{2}{c}|*{2}{c}}
         \toprule
         Feature
         &  \multicolumn{2}{c}{DenseNet121} 
         & \multicolumn{2}{c}{ResNet18} 
         & \multicolumn{2}{c}{MobileNetV2} 
        & \multicolumn{2}{c}{DeiT} 
        & \multicolumn{2}{c}{SwinT}
        \\
        Attribution
       &  $\text{INFD}\downarrow$ & $\text{FCOR}\uparrow$ 
       &   $\text{INFD}\downarrow$& $\text{FCOR}\uparrow$ &   $\text{INFD}\downarrow$& $\text{FCOR}\uparrow$ &      
       $\text{INFD}\downarrow$& $\text{FCOR}\uparrow$& $\text{INFD}\downarrow$& $\text{FCOR}\uparrow$ \\
         \midrule
        GradSHAP &2.846 & 0.303&3.196 &0.369 &0.509 &0.261 &0.512& 0.120&0.446 &0.094 \\
         IntGrad  &2.913 &0.263 &3.371 &0.312& 0.522& 0.223&0.507 &0.122 &0.444 &0.094 \\
        InputxGrad &3.097 &0.205 &3.587 &0.259& 0.538& 0.192&0.516 &0.113 &0.446 &0.079\\ 
         SmoothGrad  &2.604 &\underline{0.388} &2.916 &0.444&0.494 &0.296 &0.392 & 0.297& 0.367&0.208\\
        GradCAM &2.646 &0.388 &2.922 &\underline{0.459}&0.478 &\underline{0.319} & 0.385& 0.311& 0.373&0.227\\
        GradCAM++ &2.687 &0.376 &2.988 & 0.438 &0.484 &0.306 &0.487 & 0.213& 0.394& 0.187\\
         EigenCAM &3.044 &0.251 &3.381 &0.347 &0.538 &0.231 &0.568 &0.058 &0.439 &0.107\\
        \midrule
         $\text{AGG}_{\textit{Mean}}$  &2.661 &0.370 &2.928 & 0.444& 0.479&0.293 & 0.449& 0.237&0.377 &0.213\\
         $\text{AGG}_{\textit{Var}}$ &2.675 &0.368 &2.945 &0.442 &0.481 &0.294 &0.447 &0.238&0.381 &0.212\\
         $\text{AGG}_{\textit{robust}}$ \textbf{(ours)} &2.678 &0.339 &2.956 &0.415 &0.483 &0.282 &0.407 &0.288 &0.366 &0.224\\
         $\text{AGG}_{\textit{opt}}$ \textbf{(ours)} &\underline{2.514} &0.380&\underline{2.729} &0.458 &\underline{0.467} & 0.304& \underline{0.339} & \underline{0.368}& \underline{0.341} &\underline{0.265}\\
         $\text{AGG}_{\textit{faith}}$  \textbf{(ours)}  &\textbf{2.390} &\textbf{0.406} &\textbf{2.595} &\textbf{0.481} &\textbf{0.443} &\textbf{0.325} &\textbf{0.335} & \textbf{0.372}&\textbf{0.335} &\textbf{0.275}\\
         \bottomrule
     \end{tabular}
   }  
\end{table*}
\paragraph{Increasing faithfulness via $\text{AGG}_{\textit{faith}}$ and $\text{AGG}_{\textit{opt}}$}
Similar to the robustness experiments, we evaluate the metrics Infidelity $\text{INFD}$ and Faithfulness correlation $\text{FCOR}$ to validate the capabilities of our approach for improving attribution fidelity. For both metrics, features are perturbed by replacing randomly selected pixels with the corresponding values of a blurred image version. More details regarding the precise design of these metrics and the underlying perturbations are specified in \hyperref[app:C]{Appendix C1}. We again considered five popular model architectures including two transformer-based ones and selected seven applicable attribution methods. Table 2 summarizes the results and shows that our dedicated aggregation approach $\text{AGG}_{\textit{faith}}$ performs best in every scenario, followed again by $\text{AGG}_{\textit{opt}}$ being either second best or comparable. The consistent superiority for $\text{INFD}$ again supports our theory that optimizing aggregation using only a small amount of evaluation samples is enough to attain sustainable quality enhancement generalizing to novel out-of-sample metric evaluations unseen during optimization. Interestingly, the additional improvement concerning $\text{FCOR}$ seems to be particularly strong for the two transformer-based models. \\
\paragraph{Improving additional quality metrics}
\begin{table*}[h]
    \centering
    \caption{Additional quality metric results on a ResNet18 for different attribution methods and aggregation strategies. Across all evaluations, one of our proposed approaches performs best, and $\text{AGG}_{\textit{opt}}$ is consistently at least second best or comparable.}
    \resizebox{\textwidth}{!}{
    \begin{tabular}{c|*{3}{c}|*{2}{c}|*{2}{c}|*{3}{c}}
        \toprule
        Feature & \multicolumn{3}{c}{Stability $\downarrow$} & \multicolumn{2}{c}{Infidelity $\downarrow$}& \multicolumn{2}{c}{Faith. Corr. $\uparrow$} & \multicolumn{3}{c}{ROAD $\downarrow$} \\ 
        Attribution& RIS  & RRS & ROS  & $\text{INFD}_{0}$ & $\text{INFD}_{\bar{x}} $ & $\text{FCOR}_{0}$ & $\text{FCOR}_{\bar{x}}$ &$\text{MoRF}_{10} $& $\text{MoRF}_{20} $& $\text{MoRF}_{30} $ \\
        \midrule
        Deeplift & 8.10 & 5.26 & 8.92 & 3.59 & 3.64  & 0.33 & 0.32 & -1.10 & -2.02 & -3.02\\
        VarGrad & 7.22 & 4.73 & 7.86 & 3.18 & 3.12  & 0.42 & 0.42 & -2.37 & -4.69 & -6.79\\
        GuidedBP & \underline{3.61} & \underline{2.31} & \underline{3.91} & 3.21 & 3.10  & 0.44 & 0.45 & -3.09 & -4.86 & -6.31 \\
        IntGrad & 7.69 & 4.88 & 8.28 & 3.14 & 3.12  & 0.43 & 0.43 & -0.85 & -1.79 & -2.82\\
        SmoothGrad & 8.18 & 5.43 & 9.32 & 3.61 & 3.61 & 0.33 & 0.32 & -1.57 & -2.72 & -3.84\\
        InputxGrad & 9.47 & 6.08 & 10.28& 3.84 & 3.90  & 0.27 & 0.27& -0.63 & -1.34 & -2.24 \\
        Saliency & 9.13 & 5.89 & 9.99& 3.50 & 3.52  & 0.34 & 0.36& -0.58 & -1.24 & -2.04 \\
        \midrule
        $\text{AGG}_{\textit{Mean}}$ & 5.33 & 3.41 & 5.75 & 3.16 & 3.23  & 0.43 & 0.41 & -2.21 & -3.98 & -5.68\\
        $\text{AGG}_{\textit{Var}}$ & 5.19 & 3.33 & 5.61& 3.18 & 3.22 &  0.42 & 0.42& -2.21 & -3.98 & -5.69 \\
        $\text{AGG}_{\textit{faith}}$ \textbf{(ours)} & 5.73 & 3.70 & 6.25& \textbf{2.82} & \textbf{2.78} & \underline{0.47} & \textbf{0.48}& -2.66 & -4.50 & -6.12 \\
        $\text{AGG}_{\textit{opt}}$ \textbf{(ours)}& 3.62 & 2.33 & 3.96 & \underline{2.83} & \underline{2.80} & \textbf{0.49} & \underline{0.47}& \underline{-3.30} & \underline{-5.41} & \underline{-7.14} \\
        $\text{AGG}_{\textit{robust}}$ \textbf{(ours)} & \textbf{3.27} & \textbf{2.09} & \textbf{3.55} & 2.97 & 2.93 & 0.46 & 0.46 & \textbf{-3.36} & \textbf{-5.48} & \textbf{-7.17}\\
        \bottomrule
    \end{tabular}
    }
    \label{tab:comparison}
\end{table*}
To further substantiate the benefits of our proposed methods, we also investigated how well the improvements generalize to additional quality metrics that express alternative notions of robustness and faithfulness. Stability metrics \cite{agarwal2022rethinking} offer a complementary approach to evaluating the robustness of explanations by quantifying the sensitivity of attribution results relative to changes in various quantities of interest. More specifically, Relative Input Stability (RIS) assesses how explanations vary relative to input changes, Relative Representation Stability (RRS) examines variations relative to changes in the model's internal representations, and Relative Output Stability (ROS) evaluates sensitivity relative to changes in output prediction probabilities. The corresponding results in Table 3 over $500$ samples on a ResNet18 show that our dedicated approach $\text{AGG}_{\textit{robust}}$ also significantly improves all stability metrics. During the experiments above, we used blurring as base perturbation when optimizing aggregation for faithfulness and when evaluating Infidelity and Faithfulness correlation. To check how well the improvements generalize to slightly altering notions of faithfulness, we computed variations of these metrics based on alternative corruptions \cite{hedstrom2023quantus} such as pixel value replacement with zeros ($\text{INFD}_{0}$ and $\text{FCOR}_{0}$) and the image mean ($\text{INFD}_{\bar{x}}$ and $\text{FCOR}_{\bar{x}}$). Even though we explicitly kept blurring as perturbation during weight optimization for $\text{AGG}_{\textit{opt}}$ and $\text{AGG}_{\textit{faith}}$ the resulting explanations still perform best when evaluated with different corruptions. Finally, we computed the Remove and Debias (ROAD) metric \cite{rong2022consistent} that assesses the fidelity of explanations by removing the top features identified by an attribution method and estimating the subsequent decrease in prediction confidence. In Table 3 we report the outcomes of ablating the most relevant $p=10, 20, 30$ percent of pixels in an image ($\text{MoRF}_{p}$) while additional percentiles are deferred to \hyperref[app:D1]{Appendix D1}. All results consistently indicate that also this metric can be improved by relying on one of our proposed methods. \\
 
Overall, the results of all experiments manifest that our aggregation techniques achieve desired and generalizing improvements in explanation quality in line with our theoretical analysis. They also suggest that the aggregation strategy $\text{AGG}_{\textit{opt}}$, which optimizes for faithfulness and robustness simultaneously, is an effective default approach to attain better explainability via aggregation when both criteria matter. All these findings are confirmed on four additional datasets in \hyperref[app:D2]{Appendix D2}. Furthermore, we provide supplementary ablation studies in \hyperref[app:D4]{Appendix D4}, which imply beneficial effects resulting from an increasing number of combined explanations and greater method diversity.  

\begin{figure*}[ht!]
    \centering
\includegraphics[width=\textwidth]{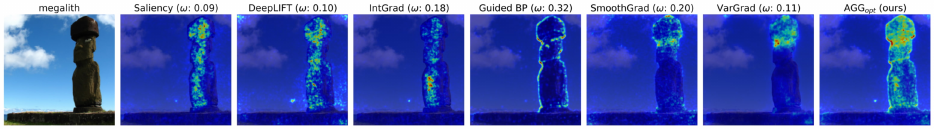}
\includegraphics[width=\textwidth]{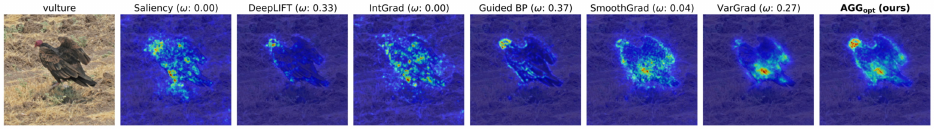}
\includegraphics[width=\textwidth]{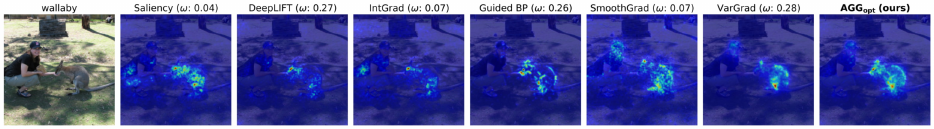}
\includegraphics[width=\textwidth]{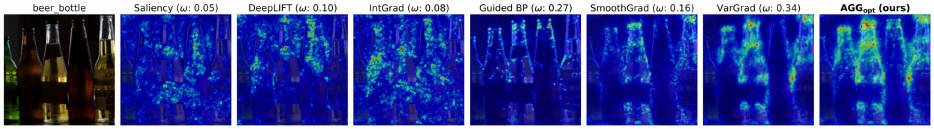}
\includegraphics[width=\textwidth]{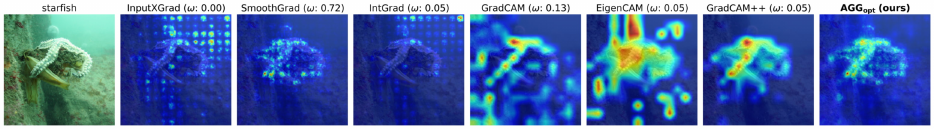}
\includegraphics[width=\textwidth]{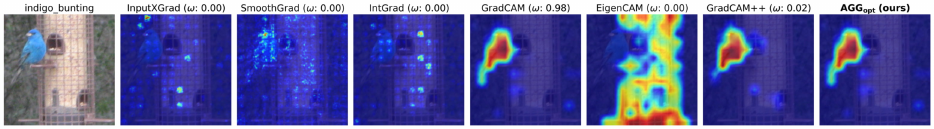}
\includegraphics[trim={0 0.8cm 0 0},width=\textwidth]{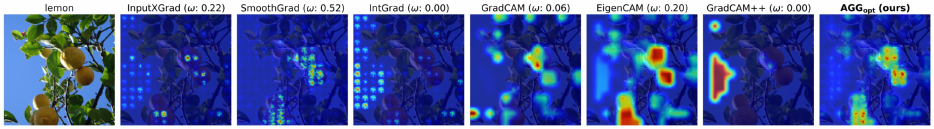}
    \caption{Individual outcomes of different feature attribution methods as well as our approach $\mathbf{\text{AGG}_{\textit{opt}}}$ \textbf{(right column)} for seven images based on VGG16 (row 1-5) and DeiT (row 6-8). In addition to the quantitative improvements established in section 5.1 for robustness and faithfulness, our aggregation strategy also produces visually more intuitive and convincing explanations. It succeeds in enhancing the attribution results by combining several valid perspectives to complement each other (e.g rows 1 and 2) and by automatically discarding seemingly deteriorated explanations (e.g. rows 7 and 8).}
    \label{results:fig}
\end{figure*} 
\begin{figure*}[h!]
    \centering
\includegraphics[trim={0 0.8cm 0 0},width=\textwidth]{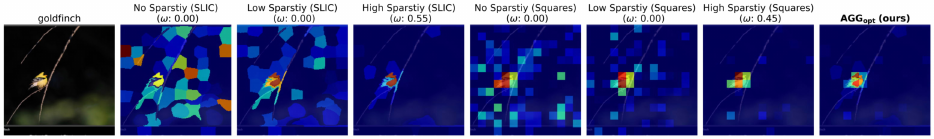}
    \caption{Individual attribution results of different LIME variants varying by superpixel structure and sparsity regularization on VGG16. The object to be classified is rather small and $\text{AGG}_{\textit{opt}}$ automatically combines only the sparsest explanations to enhance the explanation.}
    \label{lime:fig}
\end{figure*}
\subsection{Understanding how optimized aggregation helps}
In Figure 2 we display seven concrete examples with corresponding individual and aggregated attribution results to gain further insights into how optimized aggregation succeeds in improving explanations. Notice that our generic aggregation approach $\text{AGG}_{\textit{opt}}$ enhances feature attributions essentially via two mechanisms. Particularly in the first two images, all considered methods highlight intuitively relevant but diverging regions causing ambiguity about which pixels truly matter. $\text{AGG}_{\textit{opt}}$ improves the explanations by combining all perspectives to complement each other, which also leads to a visually more convincing explanation. In the last two images, some individual methods seem to fail by producing rather deteriorated results. For such instances $\text{AGG}_{\textit{opt}}$ performs automatic method selections intrinsically and aggregates only valid attribution outcomes to form an enhanced explanation that is more representative of the underlying model. Consulting the distribution of aggregation weights retrieved during both experiments in section 5.1 also reveals that the optimal weighting is highly model-dependent and even exhibits strong variability across samples. In Figure 4 we present corresponding boxplots for the weights of $\text{AGG}_{\textit{opt}}$ obtained for the two considered sets of attribution techniques. For all methods, the allocated weights during the experiments vary substantially among samples covering oftentimes even the entire possible range between 0 and 1.  Moreover, the distribution of allocated weights does not transfer across models as for different architectures other methods are most favored. This provides further evidence that a single attribution method seems unable to explain every prediction for all model architectures faithfully and supports our approach to rather aggregate them in an optimizing manner. 
\begin{figure*}[ht!]
 \begin{minipage}{0.68\textwidth}
    \includegraphics[width=\linewidth]{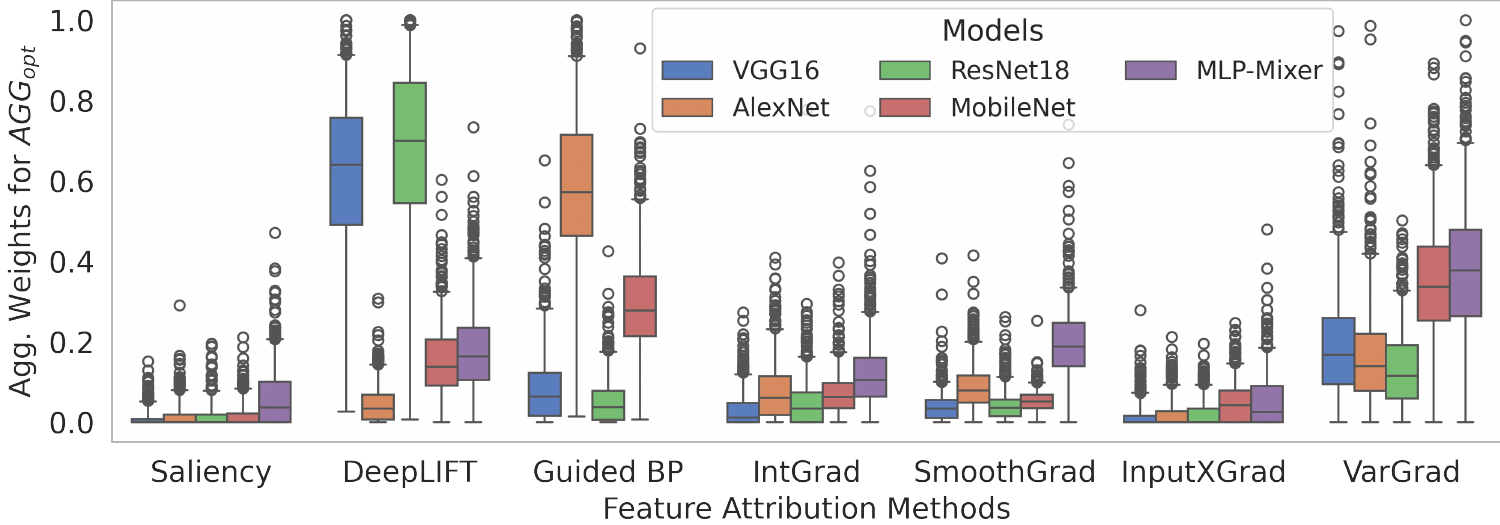}
    \vspace{0.2cm}
    \includegraphics[trim={0 1.5cm 0 0}, width=\linewidth]{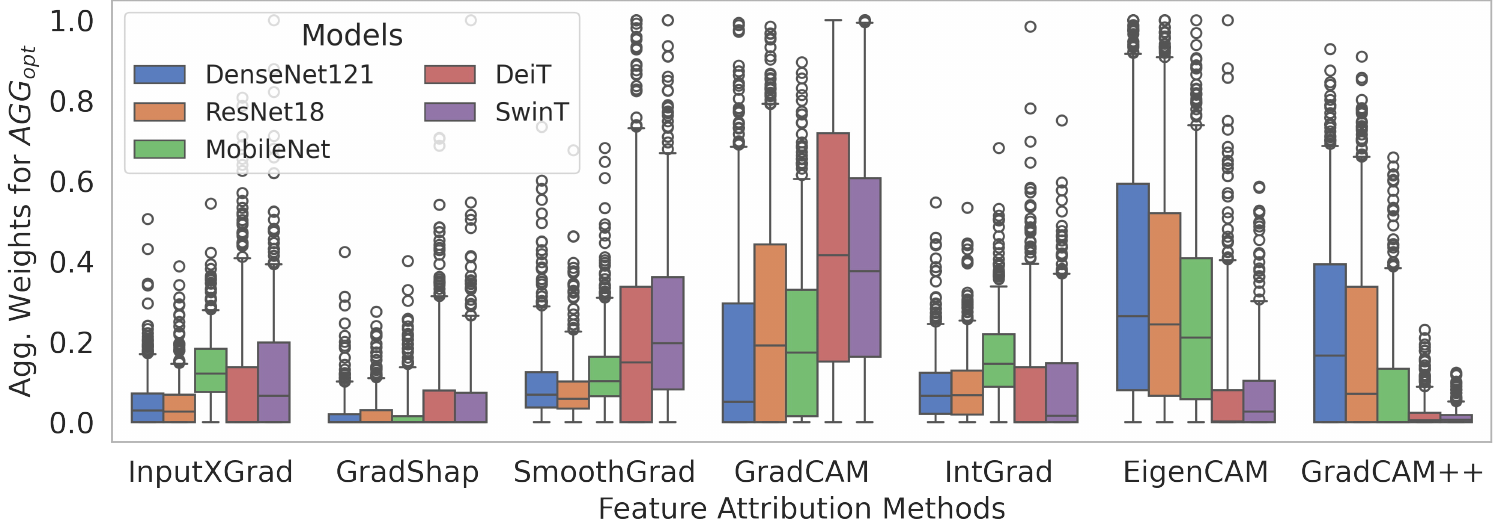}
    \caption{Boxplots of aggregation weights obtained by $\text{AGG}_{\textit{opt}}$ for the two considered sets of attribution methods during the evaluations in section 5.1 for robustness (top) and faithfulness (bottom) based on 500 samples. For each method, the allocated weight differs substantially among samples as most distributions cover almost the entire range between 0 and 1. There is also high variability across models indicating that a single method alone is unable to provide a reliable explanation for every prediction consistently.}
\end{minipage} \hfill
\begin{minipage}{0.25\textwidth}
 \includegraphics[width=\linewidth]{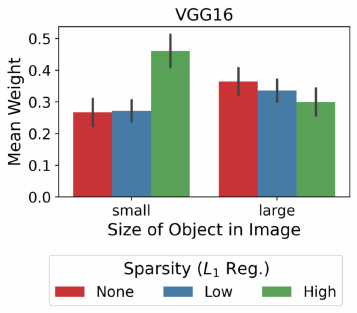}
    \includegraphics[trim={0 0.7cm 0 0}, width=\linewidth]{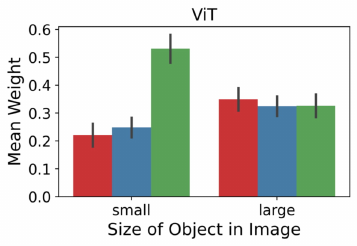}
    \caption{Average aggregation weights obtained by $\text{AGG}_{\textit{opt}}$ while optimizing the results from different versions of LIME on VGG16 (top) and ViT (bottom) including 95\% confidence intervals as error bars. For smaller objects, significantly more weight is put on higher sparsity regularization.}
\end{minipage}
\end{figure*}

\subsection{Enhancing individual methods via aggregation}
Many feature attribution methods rely on several hyperparameters and their concrete choice can greatly impact the resulting explanation \cite{bansal2020sam}. A popular example is LIME \cite{ribeiro2016should}, which derives feature importance by fitting a linear surrogate model to approximate the model behavior in the vicinity of an input. When applied to image data, LIME typically computes attributions at the level of superpixels and incorporates an $L1$ regularization to enforce a certain level of sparsity via LASSO \cite{ribeiro2016should, garreau2021does}. However, the requirement of fixing the regularization strength in advance might result in inferior explanations in cases where the number of important features does not match the enforced level of sparsity. To evaluate if optimized aggregation can effectively mediate this critical hyperparameter choice, we conducted the following experiment. We aggregate six different versions of LIME covering two different superpixel algorithms (SLIC \cite{achanta2012slic} and squared patches), each exhibiting either no, low, or high level of sparsity regularization. Using the available bounding box information for ImageNet, we distinguish between images where the object to be classified is particularly small ($<20\%$ of the total picture) or rather larger ($>60\%$). We randomly selected 200 images per object size and Figure 5 displays the average weights allocated by $\text{AGG}_{\textit{opt}}$ to different variants of LIME grouped by sparsity regularization. For both considered models, significantly more weight is put on the explanations resulting from higher sparsity regularization when the object to be classified is small compared to larger ones. This is also exemplified in Figure 3, where the prediction for an image containing a rather small object is explained. $\text{AGG}_{\textit{opt}}$ optimizes the results of all considered LIME variants by aggregating only the two sparsest attributions matching the location of interest.\\
This demonstrates that our proposed aggregation approach also boosts the performance of individual attribution techniques when combining different versions of the same method.
 
\section{Discussion and Conclusion}
In this work, we provided the first theoretically grounded approach to optimally leverage distinct feature attribution results for improving explanations of opaque models. A downside of our technique is the higher inference time especially compared to relying on a single method only. However, since the main purpose of explainability techniques is to reliably increase the transparency of particularly critical decisions we argue that the added computational costs are minor (see \hyperref[app:D3]{Appendix D3}) and well justified for the sake of provably better results. Another limiting aspect is the reliance on existing feature attribution methods and their validity as any uniform weakness might also compromise the aggregation. Hence, we recommend considering a sufficiently diverse set of individual techniques and we provide corresponding ablation studies in \hyperref[app:D4]{Appendix D4}. \\
A natural extension of our work is to consider more sophisticated strategies beyond convex weighting to perform aggregation, such as voting algorithms or other ensemble methods. Furthermore, we specifically focused on combining fairly comparable feature attribution techniques. Future work could also explore how to best incorporate supplementary insights derived from concept-based \cite{hitzler2022human}, optimization-based \cite{dabkowski2017real, fong2019understanding, jethani2021have} or counterfactual explanations \cite{guidotti2022counterfactual} to even further enhance explainability with aggregation. 
\section*{Acknowledgements}
We acknowledge the support from the German Federal Ministry for
Economic Affairs and Climate Action (BMWK) via grant agreement 19I21039A.
\section*{Impact Statement}
This paper presents work that aims at enhancing the transparency of predictions made by opaque machine learning models. Enabling more reliable explanations can be associated with a variety of societal benefits by promoting trust, accountability, and compliance with regulatory and ethical standards. Hence, any research dedicated to explainable machine learning has the potential to contribute to more responsible development and deployment of machine learning-based technology. 
\bibliography{main}

\begin{thebibliography}{74}
\providecommand{\natexlab}[1]{#1}
\providecommand{\url}[1]{\texttt{#1}}
\expandafter\ifx\csname urlstyle\endcsname\relax
  \providecommand{\doi}[1]{doi: #1}\else
  \providecommand{\doi}{doi: \begingroup \urlstyle{rm}\Url}\fi

\bibitem[Achanta et~al.(2012)Achanta, Shaji, Smith, Lucchi, Fua, and
  S{\"u}sstrunk]{achanta2012slic}
Achanta, R., Shaji, A., Smith, K., Lucchi, A., Fua, P., and S{\"u}sstrunk, S.
\newblock Slic superpixels compared to state-of-the-art superpixel methods.
\newblock \emph{IEEE transactions on pattern analysis and machine
  intelligence}, 34\penalty0 (11):\penalty0 2274--2282, 2012.

\bibitem[Adebayo et~al.(2018)Adebayo, Gilmer, Muelly, Goodfellow, Hardt, and
  Kim]{adebayo2018sanity}
Adebayo, J., Gilmer, J., Muelly, M., Goodfellow, I., Hardt, M., and Kim, B.
\newblock Sanity checks for saliency maps.
\newblock \emph{Advances in neural information processing systems}, 31, 2018.

\bibitem[Adebayo et~al.(2020)Adebayo, Muelly, Liccardi, and
  Kim]{adebayo2020debugging}
Adebayo, J., Muelly, M., Liccardi, I., and Kim, B.
\newblock Debugging tests for model explanations.
\newblock \emph{Advances in Neural Information Processing Systems},
  33:\penalty0 700--712, 2020.

\bibitem[Agarwal et~al.(2022{\natexlab{a}})Agarwal, Johnson, Pawelczyk,
  Krishna, Saxena, Zitnik, and Lakkaraju]{agarwal2022rethinking}
Agarwal, C., Johnson, N., Pawelczyk, M., Krishna, S., Saxena, E., Zitnik, M.,
  and Lakkaraju, H.
\newblock Rethinking stability for attribution-based explanations.
\newblock \emph{arXiv preprint arXiv:2203.06877}, 2022{\natexlab{a}}.

\bibitem[Agarwal et~al.(2022{\natexlab{b}})Agarwal, Krishna, Saxena, Pawelczyk,
  Johnson, Puri, Zitnik, and Lakkaraju]{agarwal2022openxai}
Agarwal, C., Krishna, S., Saxena, E., Pawelczyk, M., Johnson, N., Puri, I.,
  Zitnik, M., and Lakkaraju, H.
\newblock Open{XAI}: Towards a transparent evaluation of model explanations.
\newblock In \emph{Thirty-sixth Conference on Neural Information Processing
  Systems Datasets and Benchmarks Track}, 2022{\natexlab{b}}.

\bibitem[Alvarez-Melis \& Jaakkola(2018)Alvarez-Melis and
  Jaakkola]{alvarez2018robustness}
Alvarez-Melis, D. and Jaakkola, T.~S.
\newblock On the robustness of interpretability methods.
\newblock \emph{arXiv preprint arXiv:1806.08049}, 2018.

\bibitem[Arras et~al.(2022)Arras, Osman, and Samek]{ARRAS202214}
Arras, L., Osman, A., and Samek, W.
\newblock Clevr-xai: A benchmark dataset for the ground truth evaluation of
  neural network explanations.
\newblock \emph{Information Fusion}, 81:\penalty0 14--40, 2022.
\newblock ISSN 1566-2535.

\bibitem[Baniecki \& Biecek(2023)Baniecki and Biecek]{baniecki2023adversarial}
Baniecki, H. and Biecek, P.
\newblock Adversarial attacks and defenses in explainable artificial
  intelligence: A survey.
\newblock \emph{arXiv preprint arXiv:2306.06123}, 2023.

\bibitem[Bansal et~al.(2020)Bansal, Agarwal, and Nguyen]{bansal2020sam}
Bansal, N., Agarwal, C., and Nguyen, A.
\newblock Sam: The sensitivity of attribution methods to hyperparameters.
\newblock In \emph{Proceedings of the ieee/cvf conference on computer vision
  and pattern recognition}, pp.\  8673--8683, 2020.

\bibitem[Bhatt et~al.(2021)Bhatt, Weller, and Moura]{bhatt2021evaluating}
Bhatt, U., Weller, A., and Moura, J.~M.
\newblock Evaluating and aggregating feature-based model explanations.
\newblock In \emph{Proceedings of the Twenty-Ninth International Conference on
  International Joint Conferences on Artificial Intelligence}, pp.\
  3016--3022, 2021.

\bibitem[Bilodeau et~al.(2024)Bilodeau, Jaques, Koh, and
  Kim]{bilodeau2022impossibility}
Bilodeau, B., Jaques, N., Koh, P.~W., and Kim, B.
\newblock Impossibility theorems for feature attribution.
\newblock \emph{Proceedings of the National Academy of Sciences}, 121\penalty0
  (2):\penalty0 e2304406120, 2024.

\bibitem[Boyd \& Vandenberghe(2004)Boyd and Vandenberghe]{boyd2004convex}
Boyd, S.~P. and Vandenberghe, L.
\newblock \emph{Convex optimization}.
\newblock Cambridge university press, 2004.

\bibitem[Bykov et~al.(2022)Bykov, Hedstr{\"o}m, Nakajima, and
  H{\"o}hne]{bykov2022noisegrad}
Bykov, K., Hedstr{\"o}m, A., Nakajima, S., and H{\"o}hne, M. M.-C.
\newblock Noisegrad—enhancing explanations by introducing stochasticity to
  model weights.
\newblock In \emph{Proceedings of the AAAI Conference on Artificial
  Intelligence}, volume~36, pp.\  6132--6140, 2022.

\bibitem[Castro et~al.(2009)Castro, G{\'o}mez, and
  Tejada]{castro2009polynomial}
Castro, J., G{\'o}mez, D., and Tejada, J.
\newblock Polynomial calculation of the shapley value based on sampling.
\newblock \emph{Computers \& Operations Research}, 36\penalty0 (5):\penalty0
  1726--1730, 2009.

\bibitem[Chalasani et~al.(2020)Chalasani, Chen, Chowdhury, Wu, and
  Jha]{chalasani2020concise}
Chalasani, P., Chen, J., Chowdhury, A.~R., Wu, X., and Jha, S.
\newblock Concise explanations of neural networks using adversarial training.
\newblock In \emph{International Conference on Machine Learning}, pp.\
  1383--1391. PMLR, 2020.

\bibitem[Chattopadhay et~al.(2018)Chattopadhay, Sarkar, Howlader, and
  Balasubramanian]{chattopadhay2018grad}
Chattopadhay, A., Sarkar, A., Howlader, P., and Balasubramanian, V.~N.
\newblock Grad-cam++: Generalized gradient-based visual explanations for deep
  convolutional networks.
\newblock In \emph{2018 IEEE winter conference on applications of computer
  vision (WACV)}, pp.\  839--847. IEEE, 2018.

\bibitem[Dabkowski \& Gal(2017)Dabkowski and Gal]{dabkowski2017real}
Dabkowski, P. and Gal, Y.
\newblock Real time image saliency for black box classifiers.
\newblock \emph{Advances in neural information processing systems}, 30, 2017.

\bibitem[Decker et~al.(2023)Decker, Lebacher, and Tresp]{decker2023does}
Decker, T., Lebacher, M., and Tresp, V.
\newblock Does your model think like an engineer? explainable ai for bearing
  fault detection with deep learning.
\newblock In \emph{ICASSP 2023-2023 IEEE International Conference on Acoustics,
  Speech and Signal Processing (ICASSP)}, pp.\  1--5. IEEE, 2023.

\bibitem[Diamond \& Boyd(2016)Diamond and Boyd]{diamond2016cvxpy}
Diamond, S. and Boyd, S.
\newblock {CVXPY}: {A} {P}ython-embedded modeling language for convex
  optimization.
\newblock \emph{Journal of Machine Learning Research}, 17\penalty0
  (83):\penalty0 1--5, 2016.

\bibitem[Dicker(2014)]{dicker2014sparsity}
Dicker, L.
\newblock Sparsity and the truncated $ l^2$-norm.
\newblock In \emph{Artificial Intelligence and Statistics}, pp.\  159--166.
  PMLR, 2014.

\bibitem[Dombrowski et~al.(2019)Dombrowski, Alber, Anders, Ackermann,
  M{\"u}ller, and Kessel]{dombrowski2019explanations}
Dombrowski, A.-K., Alber, M., Anders, C., Ackermann, M., M{\"u}ller, K.-R., and
  Kessel, P.
\newblock Explanations can be manipulated and geometry is to blame.
\newblock \emph{Advances in neural information processing systems}, 32, 2019.

\bibitem[Dosovitskiy et~al.(2021)Dosovitskiy, Beyer, Kolesnikov, Weissenborn,
  Zhai, Unterthiner, Dehghani, Minderer, Heigold, Gelly, Uszkoreit, and
  Houlsby]{dosovitskiy2020image}
Dosovitskiy, A., Beyer, L., Kolesnikov, A., Weissenborn, D., Zhai, X.,
  Unterthiner, T., Dehghani, M., Minderer, M., Heigold, G., Gelly, S.,
  Uszkoreit, J., and Houlsby, N.
\newblock An image is worth 16x16 words: Transformers for image recognition at
  scale.
\newblock In \emph{International Conference on Learning Representations}, 2021.

\bibitem[Erion et~al.(2021)Erion, Janizek, Sturmfels, Lundberg, and
  Lee]{erion2021improving}
Erion, G., Janizek, J.~D., Sturmfels, P., Lundberg, S.~M., and Lee, S.-I.
\newblock Improving performance of deep learning models with axiomatic
  attribution priors and expected gradients.
\newblock \emph{Nature machine intelligence}, 3\penalty0 (7):\penalty0
  620--631, 2021.

\bibitem[Fokkema et~al.(2023)Fokkema, de~Heide, and van
  Erven]{fokkema2022attribution}
Fokkema, H., de~Heide, R., and van Erven, T.
\newblock Attribution-based explanations that provide recourse cannot be
  robust.
\newblock \emph{Journal of Machine Learning Research}, 24\penalty0
  (360):\penalty0 1--37, 2023.

\bibitem[Fong et~al.(2019)Fong, Patrick, and Vedaldi]{fong2019understanding}
Fong, R., Patrick, M., and Vedaldi, A.
\newblock Understanding deep networks via extremal perturbations and smooth
  masks.
\newblock In \emph{Proceedings of the IEEE/CVF international conference on
  computer vision}, pp.\  2950--2958, 2019.

\bibitem[Garreau \& Mardaoui(2021)Garreau and Mardaoui]{garreau2021does}
Garreau, D. and Mardaoui, D.
\newblock What does lime really see in images?
\newblock In \emph{International conference on machine learning}, pp.\
  3620--3629. PMLR, 2021.

\bibitem[Ghorbani et~al.(2019)Ghorbani, Abid, and
  Zou]{ghorbani2019interpretation}
Ghorbani, A., Abid, A., and Zou, J.
\newblock Interpretation of neural networks is fragile.
\newblock In \emph{Proceedings of the AAAI conference on artificial
  intelligence}, volume~33, pp.\  3681--3688, 2019.

\bibitem[Gildenblat(2021)]{jacobgilpytorchcam}
Gildenblat, J.
\newblock Pytorch library for cam methods.
\newblock \url{https://github.com/jacobgil/pytorch-grad-cam}, 2021.

\bibitem[Guidotti(2022)]{guidotti2022counterfactual}
Guidotti, R.
\newblock Counterfactual explanations and how to find them: literature review
  and benchmarking.
\newblock \emph{Data Mining and Knowledge Discovery}, pp.\  1--55, 2022.

\bibitem[Han et~al.(2022)Han, Srinivas, and Lakkaraju]{han2022explanation}
Han, T., Srinivas, S., and Lakkaraju, H.
\newblock Which explanation should i choose? a function approximation
  perspective to characterizing post hoc explanations.
\newblock \emph{Advances in Neural Information Processing Systems},
  35:\penalty0 5256--5268, 2022.

\bibitem[He et~al.(2016)He, Zhang, Ren, and Sun]{he2016deep}
He, K., Zhang, X., Ren, S., and Sun, J.
\newblock Deep residual learning for image recognition.
\newblock In \emph{Proceedings of the IEEE conference on computer vision and
  pattern recognition}, pp.\  770--778, 2016.

\bibitem[Hedstr{\"o}m et~al.(2023{\natexlab{a}})Hedstr{\"o}m, Weber,
  Krakowczyk, Bareeva, Motzkus, Samek, Lapuschkin, and
  H{\"o}hne]{hedstrom2023quantus}
Hedstr{\"o}m, A., Weber, L., Krakowczyk, D., Bareeva, D., Motzkus, F., Samek,
  W., Lapuschkin, S., and H{\"o}hne, M. M.-C.
\newblock Quantus: An explainable ai toolkit for responsible evaluation of
  neural network explanations and beyond.
\newblock \emph{Journal of Machine Learning Research}, 24\penalty0
  (34):\penalty0 1--11, 2023{\natexlab{a}}.

\bibitem[Hedstr{\"o}m et~al.(2023{\natexlab{b}})Hedstr{\"o}m, Weber,
  Lapuschkin, and H{\"o}hne]{hedstrom2023sanity}
Hedstr{\"o}m, A., Weber, L., Lapuschkin, S., and H{\"o}hne, M.
\newblock Sanity checks revisited: An exploration to repair the model parameter
  randomisation test.
\newblock In \emph{XAI in Action: Past, Present, and Future Applications},
  2023{\natexlab{b}}.

\bibitem[Hitzler \& Sarker(2022)Hitzler and Sarker]{hitzler2022human}
Hitzler, P. and Sarker, M.
\newblock Human-centered concept explanations for neural networks.
\newblock \emph{Neuro-Symbolic Artificial Intelligence: The State of the Art},
  342\penalty0 (337):\penalty0 2, 2022.

\bibitem[Hooker et~al.(2019)Hooker, Erhan, Kindermans, and
  Kim]{hooker2019benchmark}
Hooker, S., Erhan, D., Kindermans, P.-J., and Kim, B.
\newblock A benchmark for interpretability methods in deep neural networks.
\newblock \emph{Advances in neural information processing systems}, 32, 2019.

\bibitem[Huang et~al.(2017)Huang, Liu, Van Der~Maaten, and
  Weinberger]{huang2017densely}
Huang, G., Liu, Z., Van Der~Maaten, L., and Weinberger, K.~Q.
\newblock Densely connected convolutional networks.
\newblock In \emph{Proceedings of the IEEE conference on computer vision and
  pattern recognition}, pp.\  4700--4708, 2017.

\bibitem[Jethani et~al.(2021)Jethani, Sudarshan, Aphinyanaphongs, and
  Ranganath]{jethani2021have}
Jethani, N., Sudarshan, M., Aphinyanaphongs, Y., and Ranganath, R.
\newblock Have we learned to explain?: How interpretability methods can learn
  to encode predictions in their interpretations.
\newblock In \emph{International Conference on Artificial Intelligence and
  Statistics}, pp.\  1459--1467. PMLR, 2021.

\bibitem[Kindermans et~al.(2019)Kindermans, Hooker, Adebayo, Alber, Sch{\"u}tt,
  D{\"a}hne, Erhan, and Kim]{kindermans2019reliability}
Kindermans, P.-J., Hooker, S., Adebayo, J., Alber, M., Sch{\"u}tt, K.~T.,
  D{\"a}hne, S., Erhan, D., and Kim, B.
\newblock The (un) reliability of saliency methods.
\newblock In \emph{Explainable AI: Interpreting, Explaining and Visualizing
  Deep Learning}, pp.\  267--280. Springer, 2019.

\bibitem[Kokhlikyan et~al.(2020)Kokhlikyan, Miglani, Martin, Wang, Alsallakh,
  Reynolds, Melnikov, Kliushkina, Araya, Yan, et~al.]{kokhlikyan2020captum}
Kokhlikyan, N., Miglani, V., Martin, M., Wang, E., Alsallakh, B., Reynolds, J.,
  Melnikov, A., Kliushkina, N., Araya, C., Yan, S., et~al.
\newblock Captum: A unified and generic model interpretability library for
  pytorch.
\newblock \emph{arXiv preprint arXiv:2009.07896}, 2020.

\bibitem[Krishna et~al.(2022)Krishna, Han, Gu, Pombra, Jabbari, Wu, and
  Lakkaraju]{krishna2022disagreement}
Krishna, S., Han, T., Gu, A., Pombra, J., Jabbari, S., Wu, S., and Lakkaraju,
  H.
\newblock The disagreement problem in explainable machine learning: A
  practitioner's perspective.
\newblock \emph{arXiv preprint arXiv:2202.01602}, 2022.

\bibitem[Krizhevsky et~al.(2012)Krizhevsky, Sutskever, and
  Hinton]{krizhevsky2012imagenet}
Krizhevsky, A., Sutskever, I., and Hinton, G.~E.
\newblock Imagenet classification with deep convolutional neural networks.
\newblock \emph{Advances in neural information processing systems}, 25, 2012.

\bibitem[Krogh \& Vedelsby(1994)Krogh and Vedelsby]{krogh1994neural}
Krogh, A. and Vedelsby, J.
\newblock Neural network ensembles, cross validation, and active learning.
\newblock \emph{Advances in neural information processing systems}, 7, 1994.

\bibitem[Kumar et~al.(2021)Kumar, Scheidegger, Venkatasubramanian, and
  Friedler]{kumar2021shapley}
Kumar, I., Scheidegger, C., Venkatasubramanian, S., and Friedler, S.
\newblock Shapley residuals: Quantifying the limits of the shapley value for
  explanations.
\newblock \emph{Advances in Neural Information Processing Systems},
  34:\penalty0 26598--26608, 2021.

\bibitem[Lin et~al.(2023)Lin, Covert, and Lee]{lin2023robustness}
Lin, C., Covert, I., and Lee, S.-I.
\newblock On the robustness of removal-based feature attributions.
\newblock \emph{arXiv preprint arXiv:2306.07462}, 2023.

\bibitem[Liu et~al.(2021)Liu, Lin, Cao, Hu, Wei, Zhang, Lin, and
  Guo]{liu2021swin}
Liu, Z., Lin, Y., Cao, Y., Hu, H., Wei, Y., Zhang, Z., Lin, S., and Guo, B.
\newblock Swin transformer: Hierarchical vision transformer using shifted
  windows.
\newblock In \emph{Proceedings of the IEEE/CVF international conference on
  computer vision}, pp.\  10012--10022, 2021.

\bibitem[Lundberg \& Lee(2017)Lundberg and Lee]{lundberg2017unified}
Lundberg, S.~M. and Lee, S.-I.
\newblock A unified approach to interpreting model predictions.
\newblock \emph{Advances in neural information processing systems}, 30, 2017.

\bibitem[Maurer(2016)]{maurer2016vector}
Maurer, A.
\newblock A vector-contraction inequality for rademacher complexities.
\newblock In \emph{Algorithmic Learning Theory: 27th International Conference,
  ALT 2016 Proceedings 27}, pp.\  3--17. Springer, 2016.

\bibitem[Muhammad \& Yeasin(2020)Muhammad and Yeasin]{muhammad2020eigen}
Muhammad, M.~B. and Yeasin, M.
\newblock Eigen-cam: Class activation map using principal components.
\newblock In \emph{2020 international joint conference on neural networks
  (IJCNN)}, pp.\  1--7. IEEE, 2020.

\bibitem[Nauta et~al.(2023)Nauta, Trienes, Pathak, Nguyen, Peters, Schmitt,
  Schl{\"o}tterer, van Keulen, and Seifert]{nauta2023anecdotal}
Nauta, M., Trienes, J., Pathak, S., Nguyen, E., Peters, M., Schmitt, Y.,
  Schl{\"o}tterer, J., van Keulen, M., and Seifert, C.
\newblock From anecdotal evidence to quantitative evaluation methods: A
  systematic review on evaluating explainable ai.
\newblock \emph{ACM Computing Surveys}, 55\penalty0 (13s):\penalty0 1--42,
  2023.

\bibitem[Neely et~al.(2021)Neely, Schouten, Bleeker, and Lucic]{neely2021order}
Neely, M., Schouten, S.~F., Bleeker, M.~J., and Lucic, A.
\newblock Order in the court: Explainable ai methods prone to disagreement.
\newblock \emph{arXiv preprint arXiv:2105.03287}, 2021.

\bibitem[Nie et~al.(2018)Nie, Zhang, and Patel]{nie2018theoretical}
Nie, W., Zhang, Y., and Patel, A.
\newblock A theoretical explanation for perplexing behaviors of
  backpropagation-based visualizations.
\newblock In \emph{International conference on machine learning}, pp.\
  3809--3818. PMLR, 2018.

\bibitem[Pahde et~al.(2023)Pahde, Yolcu, Binder, Samek, and
  Lapuschkin]{Pahde_2023_CVPR}
Pahde, F., Yolcu, G.~U., Binder, A., Samek, W., and Lapuschkin, S.
\newblock Optimizing explanations by network canonization and hyperparameter
  search.
\newblock In \emph{Proceedings of the IEEE/CVF Conference on Computer Vision
  and Pattern Recognition (CVPR) Workshops}, pp.\  3819--3828, June 2023.

\bibitem[Rebuffi et~al.(2020)Rebuffi, Fong, Ji, and Vedaldi]{rebuffi2020there}
Rebuffi, S.-A., Fong, R., Ji, X., and Vedaldi, A.
\newblock There and back again: Revisiting backpropagation saliency methods.
\newblock In \emph{Proceedings of the IEEE/CVF Conference on Computer Vision
  and Pattern Recognition}, pp.\  8839--8848, 2020.

\bibitem[Ribeiro et~al.(2016)Ribeiro, Singh, and Guestrin]{ribeiro2016should}
Ribeiro, M.~T., Singh, S., and Guestrin, C.
\newblock " why should i trust you?" explaining the predictions of any
  classifier.
\newblock In \emph{Proceedings of the 22nd ACM SIGKDD international conference
  on knowledge discovery and data mining}, pp.\  1135--1144, 2016.

\bibitem[Rieger \& Hansen(2019)Rieger and Hansen]{rieger2019aggregating}
Rieger, L. and Hansen, L.~K.
\newblock Aggregating explanation methods for stable and robust explainability.
\newblock \emph{arXiv preprint arXiv:1903.00519}, 2019.

\bibitem[Rong et~al.(2022)Rong, Leemann, Borisov, Kasneci, and
  Kasneci]{rong2022consistent}
Rong, Y., Leemann, T., Borisov, V., Kasneci, G., and Kasneci, E.
\newblock A consistent and efficient evaluation strategy for attribution
  methods.
\newblock In \emph{International Conference on Machine Learning}, pp.\
  18770--18795. PMLR, 2022.

\bibitem[Sandler et~al.(2018)Sandler, Howard, Zhu, Zhmoginov, and
  Chen]{sandler2018mobilenetv2}
Sandler, M., Howard, A., Zhu, M., Zhmoginov, A., and Chen, L.-C.
\newblock Mobilenetv2: Inverted residuals and linear bottlenecks.
\newblock In \emph{Proceedings of the IEEE conference on computer vision and
  pattern recognition}, pp.\  4510--4520, 2018.

\bibitem[Selvaraju et~al.(2017)Selvaraju, Cogswell, Das, Vedantam, Parikh, and
  Batra]{selvaraju2017grad}
Selvaraju, R.~R., Cogswell, M., Das, A., Vedantam, R., Parikh, D., and Batra,
  D.
\newblock Grad-cam: Visual explanations from deep networks via gradient-based
  localization.
\newblock In \emph{Proceedings of the IEEE international conference on computer
  vision}, pp.\  618--626, 2017.

\bibitem[Shalev-Shwartz \& Ben-David(2014)Shalev-Shwartz and
  Ben-David]{shalev2014understanding}
Shalev-Shwartz, S. and Ben-David, S.
\newblock \emph{Understanding machine learning: From theory to algorithms}.
\newblock Cambridge university press, 2014.

\bibitem[Shrikumar et~al.(2016)Shrikumar, Greenside, Shcherbina, and
  Kundaje]{shrikumar2016not}
Shrikumar, A., Greenside, P., Shcherbina, A., and Kundaje, A.
\newblock Not just a black box: Learning important features through propagating
  activation differences.
\newblock \emph{arXiv preprint arXiv:1605.01713}, 2016.

\bibitem[Shrikumar et~al.(2017)Shrikumar, Greenside, and
  Kundaje]{shrikumar2017learning}
Shrikumar, A., Greenside, P., and Kundaje, A.
\newblock Learning important features through propagating activation
  differences.
\newblock In \emph{International conference on machine learning}, pp.\
  3145--3153. PMLR, 2017.

\bibitem[Simonyan et~al.(2013)Simonyan, Vedaldi, and
  Zisserman]{simonyan2013deep}
Simonyan, K., Vedaldi, A., and Zisserman, A.
\newblock Deep inside convolutional networks: Visualising image classification
  models and saliency maps.
\newblock \emph{arXiv preprint arXiv:1312.6034}, 2013.

\bibitem[Sixt et~al.(2020)Sixt, Granz, and Landgraf]{sixt2020explanations}
Sixt, L., Granz, M., and Landgraf, T.
\newblock When explanations lie: Why many modified bp attributions fail.
\newblock In \emph{International Conference on Machine Learning}, pp.\
  9046--9057. PMLR, 2020.

\bibitem[Smilkov et~al.(2017)Smilkov, Thorat, Kim, Vi{\'e}gas, and
  Wattenberg]{smilkov2017smoothgrad}
Smilkov, D., Thorat, N., Kim, B., Vi{\'e}gas, F., and Wattenberg, M.
\newblock Smoothgrad: removing noise by adding noise.
\newblock \emph{arXiv preprint arXiv:1706.03825}, 2017.

\bibitem[Springenberg et~al.(2014)Springenberg, Dosovitskiy, Brox, and
  Riedmiller]{springenberg2014striving}
Springenberg, J.~T., Dosovitskiy, A., Brox, T., and Riedmiller, M.
\newblock Striving for simplicity: The all convolutional net.
\newblock \emph{arXiv preprint arXiv:1412.6806}, 2014.

\bibitem[Sturmfels et~al.(2020)Sturmfels, Lundberg, and
  Lee]{sturmfels2020visualizing}
Sturmfels, P., Lundberg, S., and Lee, S.-I.
\newblock Visualizing the impact of feature attribution baselines.
\newblock \emph{Distill}, 5\penalty0 (1):\penalty0 e22, 2020.

\bibitem[Sundararajan et~al.(2017)Sundararajan, Taly, and
  Yan]{sundararajan2017axiomatic}
Sundararajan, M., Taly, A., and Yan, Q.
\newblock Axiomatic attribution for deep networks.
\newblock In \emph{International conference on machine learning}, pp.\
  3319--3328. PMLR, 2017.

\bibitem[Tolstikhin et~al.(2021)Tolstikhin, Houlsby, Kolesnikov, Beyer, Zhai,
  Unterthiner, Yung, Steiner, Keysers, Uszkoreit, et~al.]{tolstikhin2021mlp}
Tolstikhin, I.~O., Houlsby, N., Kolesnikov, A., Beyer, L., Zhai, X.,
  Unterthiner, T., Yung, J., Steiner, A., Keysers, D., Uszkoreit, J., et~al.
\newblock Mlp-mixer: An all-mlp architecture for vision.
\newblock \emph{Advances in neural information processing systems},
  34:\penalty0 24261--24272, 2021.

\bibitem[Touvron et~al.(2021)Touvron, Cord, Douze, Massa, Sablayrolles, and
  J{\'e}gou]{touvron2021training}
Touvron, H., Cord, M., Douze, M., Massa, F., Sablayrolles, A., and J{\'e}gou,
  H.
\newblock Training data-efficient image transformers \& distillation through
  attention.
\newblock In \emph{International conference on machine learning}, pp.\
  10347--10357. PMLR, 2021.

\bibitem[Wang et~al.(2020)Wang, Wang, Ramkumar, Mardziel, Fredrikson, and
  Datta]{wang2020smoothed}
Wang, Z., Wang, H., Ramkumar, S., Mardziel, P., Fredrikson, M., and Datta, A.
\newblock Smoothed geometry for robust attribution.
\newblock \emph{Advances in neural information processing systems},
  33:\penalty0 13623--13634, 2020.

\bibitem[Wightman(2019)]{rw2019timm}
Wightman, R.
\newblock Pytorch image models.
\newblock \url{https://github.com/rwightman/pytorch-image-models}, 2019.

\bibitem[Yang et~al.(2023)Yang, Shi, Wei, Liu, Zhao, Ke, Pfister, and
  Ni]{yang2023medmnist}
Yang, J., Shi, R., Wei, D., Liu, Z., Zhao, L., Ke, B., Pfister, H., and Ni, B.
\newblock Medmnist v2-a large-scale lightweight benchmark for 2d and 3d
  biomedical image classification.
\newblock \emph{Scientific Data}, 10\penalty0 (1):\penalty0 41, 2023.

\bibitem[Yeh et~al.(2019)Yeh, Hsieh, Suggala, Inouye, and
  Ravikumar]{yeh2019fidelity}
Yeh, C.-K., Hsieh, C.-Y., Suggala, A., Inouye, D.~I., and Ravikumar, P.~K.
\newblock On the (in) fidelity and sensitivity of explanations.
\newblock \emph{Advances in Neural Information Processing Systems}, 32, 2019.

\bibitem[Zhou et~al.(2022)Zhou, Booth, Ribeiro, and Shah]{zhou2022feature}
Zhou, Y., Booth, S., Ribeiro, M.~T., and Shah, J.
\newblock Do feature attribution methods correctly attribute features?
\newblock In \emph{Proceedings of the AAAI Conference on Artificial
  Intelligence}, volume~36, pp.\  9623--9633, 2022.

\end{thebibliography}
\bibliographystyle{icml2024}

\newpage
\appendix
\onecolumn
\section{Theoretical Proofs}
\label{app:A}
In this section, we conduct the proofs of the theoretical results presented in the main paper.
\paragraph{Probable Improvements via Aggregation}
\begin{theorem}
    Let $\phi^{\omega}=\sum_i \omega_i \phi^i$ be the aggregated explanation then the quality of $\phi^{\omega}$ is better than the the weighted quality of the individual attribution results:
    \begin{align*}
    \mathcal{Q}(\phi^{\omega}) = \sum_i \omega_i \mathcal{Q}(\phi^i) - \mathbb{E}_{\gamma_1}\left[ \sum_i \omega_i \lVert \gamma_1 (\phi_i- \phi^{\omega})\rVert_2^2\right]
    \end{align*}
\end{theorem}
\begin{proof}
The proof is similar to a related result established in \cite{krogh1994neural}:
\begin{align*}
    \sum_{i=1}^k \omega_i \mathcal{Q}(\phi^i) -\mathcal{Q}(\phi^{\omega}) &= \sum_{i=1}^k \omega_i \mathbb{E}\left[ \lVert \gamma_1 \phi^i - \gamma_2 \lVert_2^2\right] -  \mathbb{E}\left[ \lVert  \gamma_1 \phi^{\omega} - \gamma_2 \lVert_2^2\right] 
    = \mathbb{E}\left[\sum_{i=1}^k \omega_i \lVert \gamma_1 \phi^{i} - \gamma_2 \lVert_2^2 \; - \; \lVert  \gamma_1 \phi^{\omega} - \gamma_2 \lVert_2^2 \right]\\
    &= \mathbb{E}\left[\sum_{i=1}^k \omega_i \left( {\phi^{i}}^T \gamma_1^T\gamma_1 \phi^{i} -2\gamma_2^T\gamma_1 \phi^{i}  + \gamma_2^T\gamma_2 \right) \; - \; {\phi^{\omega}}^T \gamma_1^T\gamma_1 \phi^{\omega} -2\gamma_2^T\gamma_1 \phi^{\omega}  + \gamma_2^T\gamma_2 \right]\\
    &= \mathbb{E}\left[\sum_{i=1}^k \omega_i \left( {\phi^{i}}^T \gamma_1^T\gamma_1 \phi^{i}\right) \; - \; {\phi^{\omega}}^T \gamma_1^T\gamma_1 \phi^{\omega} \right] = \mathbb{E}\left[ \sum_{i=1}^k \omega_i \left(\gamma_1 (\phi^i - \phi^{\omega}) \right)^T\left(\gamma_1 (\phi^i - \phi^{\omega})\right)\right]\\
    &=\mathbb{E}\left[ \sum_{i=1}^k \omega_i \lVert \gamma_1 (\phi_i- \phi^{\omega})\rVert_2^2\right]
\end{align*}
\end{proof}
\paragraph{Generalization Bound}
The proof of the generalization bound leverages the following result from \cite{shalev2014understanding}, where we slightly adapted the notation to better match our setup:
\begin{theorem}[26.5.3 in \cite{shalev2014understanding} ]
Let $(\mathcal{X} \times \mathcal{Y})$ be a probability space and $\ell: \mathcal{X} \times \mathcal{Y} \rightarrow \mathbb{R}$ be a bounded loss function with $\ell(x, y) \le L$. Let $\hat{f} = \arg\min_{f \in \mathcal{F}} \frac{1}{m}\sum_{i=1}^m \ell(x_i, y_i)$ be an empirical estimator for the minimum of $\mathcal{L}(f)=\mathbb{E}\left[ \ell(f(X), Y )\right]$. Also,for Rademacher variables $\varepsilon_i \in \{-1, 1\}$ the empirical Rademacher Complexity of a function set $\mathcal{F}$ is defined by:
\begin{align*}
    \hat{\mathcal{R}}_m(\mathcal{F}) = \frac{1}{m} \mathbb{E}_{\varepsilon}\left[ \sup_{f \in \mathcal{F}} \sum_{i=1}^m \varepsilon_i f(x_i) \right]
\end{align*}
Then, it holds with a probability of at least $1-\delta$:
\begin{align*} 
    \mathcal{L}(\hat{f}) - \min_{f \in \mathcal{F}} \mathcal{L}(f) \le 2 \hat{\mathcal{R}}_m(\ell \circ \mathcal{F}) + 5 L \sqrt{\frac{2\ln(8/\delta)}{m}}
\end{align*}
\end{theorem}
On top of that, we need the following Lemmas:
\begin{lemma} Suppose that for all $x, y \in S \subset \mathbb{R}^g$ we have $\lVert x -y \rVert_2 \le c$ for a constant $c$. 
Then, the squared Euclidean distance $l(x, y) = \lVert x - y \rVert_2^2$ is Lipschitz continuous in its first argument with Lipschitz constant $L=2c$.
\end{lemma}
\begin{proof}
The statement follows from quadratic factorization and the reverse triangular inequality. \\
For all $x, x', y \in S$ it holds:
\begin{align*}
    \lvert l(x, y) - l(x', y)\rvert &=\lvert  \lVert x - y \rVert_2^2 -\lVert x' - y \rVert_2^2 \vert \\ &= \lvert \left( \lVert x - y \rVert_2 + \lVert x' - y \rVert_2 \right) \left( \lVert x - y \rVert_2 - \lVert x' - y \rVert_2 \right) \vert \\
    & \le \lvert \left( \lVert x - y \rVert_2 + \lVert x' - y \rVert_2 \right)\rvert \; \lVert x-y - x'+ y \rVert_2\\
    & \le 2c \; \lVert x - x' \rVert_2
\end{align*}
\end{proof}

\begin{lemma} For $i=1, \dots, m$, let $x_i\in \mathbb{R}^k$ and $\varepsilon_i \in \{-1,1\}$ be Rademacher variables, so $\mathbb{P}(\varepsilon=\pm1)=1/2$. Then:
\begin{align*}
\mathbb{E}_{\varepsilon}\left[ \lVert \sum_{i=1}^m \varepsilon_ix_i \rVert_{\infty}\right]\le \sqrt{m}\max_{i=1, \dots, m} \lVert x_i \rVert_{\infty} \sqrt{2\ln(2k)}
\end{align*}
\end{lemma}
\begin{proof}
The proof leverages Massart's lemma and is for instance conducted within the proof of Lemma 26.11  in \cite{shalev2014understanding}
\end{proof}
Now we are equipped to proof Theorem 4.3:
\begin{theorem}  
Let $\mathcal{Q}$ be a generalized $L2$ metric with $\max_{\gamma_1} \lVert \gamma_1 \rVert_1 \le c_1$ and let $\Phi = (\phi^{1},\dots, \phi^{k})$ be the matrix of stacked attribution outcomes to be aggregated into $\phi^{\omega} = \sum_{i=1}^k \omega_i \phi^i$. Suppose that $\max_{\gamma_1, \gamma_2} \lVert \gamma_1\phi^i  -\gamma_2 \rVert_2^2 \le c_2$  as well as $\lVert \phi^i\rVert_{\infty}\le 1$ for all $i=1, \dots, k$. 
Also let $\Omega$ be the set of feasible weights $\omega$ and let $\hat{\omega}$ be an aggregation weight estimate obtained from $m$ metric evaluations given by
\begin{align*}
    \hat{\omega} =& \arg\min_{\omega \in \Omega}\; \; \frac{1}{m}\sum_{j=1}^m \; \; \lVert \gamma_1^{(j)}\phi^{\omega} - \gamma_2^{(j)} \rVert_2^2 
\end{align*}
Then there exist a constant $C(c_1, c_2) >0 $ depending on $c_1$ and $c_2$ such that with probability of at least $(1-\delta)$:
\begin{align*}
\mathcal{Q}(\phi^{\hat{\omega}}) - \min_{\omega \in \Omega} \mathcal{Q}(\phi^{\omega})  \le C \; \sqrt{\dfrac{\ln(16 k/\delta)}{m}}
\end{align*}
\end{theorem}
\begin{proof} The theorem can be interpreted as an extension and adaptation of Theorem 26.15 in \cite{shalev2014understanding} to the specifics of our setup. We develop appropriate bounds on the Rademacher complexity of vector-valued functions based on a concentration result from \cite{maurer2016vector} and the specific properties of generalized $L2$ metrics over convex combinations of normalized feature attribution results.\\
Let $\mathcal{F}= \{f :\mathbb{R}^{g\times k} \rightarrow \mathbb{R}^g, \; f(A)= A\omega \;| \; \omega \in \Omega \}$,  then with Theorem A2 above we immediately get
\begin{align}
    \mathcal{Q}(\phi^{\hat{\omega}}) - \min_{\omega \in \Omega} \mathcal{Q}(\phi^{\omega}) \le 2 \hat{\mathcal{R}}_m(\ell \circ \mathcal{F}) + 5 c_2 \sqrt{\frac{2\ln(8/\delta)}{m}}
\end{align}
where $l$ is the squared Euclidean distance. To ease the notation in the following, define $A^{(i)} := \gamma_1^{(i)}\Phi $ and the $j$-th row of $A^{(i)}$ as $A^{(i)}_{j:} \in \mathbb{R}^k$. Using the assumption that $\max_{\gamma_1, \gamma_2} \lVert \gamma_1\phi^i  -\gamma_2 \rVert_2^2 \le c_2$, we know from Lemma A.3 that $l$ is Lipschitz continuous with constant $L = 2\sqrt{c_2}$.
Therefore, we are able to leverage a corresponding result from \cite{maurer2016vector} $(\star)$ to upper-bound the empirical Rademacher complexity of the vector-valued function set. More precisely, it holds:
\begin{align*}
m\hat{\mathcal{R}}_m(\ell \circ \mathcal{F}) &= \mathbb{E}_{\varepsilon}\left[ \sup_{f \in \mathcal{F}} \sum_{i=1}^m \varepsilon_i \lVert f(A^{(i)}) -\gamma_2^{(i)}\lVert_2^2 \right] \overset{(\star)}{\le} \sqrt{2} L  \mathbb{E}_{\varepsilon}\left[ \sup_{f \in \mathcal{F}} \sum_{i=1}^m \sum_{j=1}^g \varepsilon_{i,j} f_j(A^{(i)}) \right] \\
 &= 2\sqrt{2c_2}  \mathbb{E}_{\varepsilon}\left[ \sup_{\omega \in \Omega} \sum_{i=1}^m \sum_{j=1}^g \varepsilon_{i,j} (A^{(i)}\omega)_j \right] = 2\sqrt{2c_2}  \mathbb{E}_{\varepsilon}\left[ \sup_{\omega \in \Omega} \left\langle \sum_{i=1}^m \sum_{j=1}^g \varepsilon_{i,j} A^{(i)}_{j:}, \omega\right\rangle \right]\\
 &\le 2\sqrt{2c_2}  \mathbb{E}_{\varepsilon}\left[ \sup_{\omega \in \Omega} \lVert \sum_{i=1}^m \sum_{j=1}^g \varepsilon_{i,j} A^{(i)}_{j:}\rVert_{\infty} \left\lVert \omega \right\rVert_1 \right] =  2\sqrt{2c_2}  \mathbb{E}_{\varepsilon}\left[ \lVert \sum_{i=1}^m \sum_{j=1}^g \varepsilon_{i,j} A^{(i)}_{j:}\rVert_{\infty}\right]
\end{align*}
where $\varepsilon_i$ as well as $\varepsilon_{i,j}$ are Rademacher variables and the last two steps follow from the Hölder inequality as well as the constraints on $\omega$. Next, notice that the term $\sum_{i=1}^m \sum_{j=1}^g \varepsilon_{i,j} A^{(i)}_{j:}$ sums over all row of $A^{i}$ across all samples. Hence we can reindex the term as a sum over all consecutive rows in the sample denoted by $a^l \in \mathbb{R}^k$ with $l=1,\dots, gm$.\\ Applying Lemma A.4 yields:
\begin{align*}
    \mathbb{E}_{\varepsilon}\left[ \lVert \sum_{i=1}^m \sum_{j=1}^g \varepsilon_{i,j} A^{(i)}_{j:}\rVert_{\infty}\right] = \mathbb{E}_{\varepsilon}\left[ \lVert \sum_{l=1}^{gm} \varepsilon_{l} a^l\rVert_{\infty}\right] \le \sqrt{gm}\max_{l=1, \dots, gm} \lVert a^l \rVert_{\infty} \sqrt{2\ln(2k)}
\end{align*}
Notice that bounding $\max_{l=1, \dots, gm} \lVert a^l \rVert_{\infty}$ requires to bound the maximal entry of $\gamma_1\Phi$ that could be encountered while computing the metric. Using the assumed constraints on $\gamma_1$ and $\Phi$ we obtain:
\begin{align*}
\max_{l=1, \dots, gm} \lVert a^l \rVert_{\infty}  \le  \max_{i,j} \lvert (\gamma_1\Phi)_{i,j} \rvert \le \max_{\gamma_1} \lVert \gamma_1 \lVert_1 \; \max_j \lVert \phi^j \rVert_{\infty} \le c_1
\end{align*}
Therefore, we finally have an upper bound on the empirical Rademacher complexity given by:
\begin{align}
    \hat{\mathcal{R}}_m(\ell \circ \mathcal{F}) \le 2c_1\sqrt{2gc_2}\sqrt{\dfrac{2\ln(2k)}{m}}
\end{align}
Combining $(1)$ and $(2)$  and setting $C:= \max\{4c_1\sqrt{2gc_2},5c_2\}$ gives:
\begin{align*}
        \mathcal{Q}(\phi^{\hat{\omega}}) - \min_{\omega \in \Omega} \mathcal{Q}(\phi^{\omega}) \le 4c_1\sqrt{2gc_2}\sqrt{\dfrac{2\ln(2k)}{m}}  + 5 c_2 \sqrt{\frac{2\ln(8/\delta)}{m}}
        \le C\sqrt{\dfrac{4\ln(16k/\delta)}{m}}
\end{align*}
where the last step utilizes the fact that $\sqrt{a} + \sqrt{b} \le \sqrt{2(a+b)}$ to merge the two square roots.
\end{proof}

\section{Additional Generalized L2 Metrics for other Dimensions of Explanation Quality}
\label{app:B}
In the main paper, we showed that popular metrics for feature attribution such as Infidelity and Average-Sensitivity are generalized $L2$ metrics. Below we also show that metrics regarding other quality criteria can be expressed as such.
\paragraph{Alignment Metrics}
Alignment metrics, also referred to as localization metrics, measure to which extent an attribution result corresponds to an expected explanation grounded in domain knowledge. For image classification models, such metrics typically quantify how well important image regions overlap with the actual location of the classified object in the image. A simple way to achieve this is to define $\gamma_2^{\star}$ as desired attribution results and measure alignment via the squared Euclidean distance: $\mathcal{Q}(\phi(x)) = \lVert \phi(x) -\gamma_2^{\star} \rVert_2^2$. Another possibility that closely resembles the logic of localization metrics for computer vision models is to measure if important features lie within a region of interest. Let $\mathcal{I} \subset \{1, \dots d\}$ be an index indicating the position of an object to be detected by a model. Then $\mathcal{Q}(\phi(x)) = \lVert \phi(x) -\phi_{\mathcal{I}}(x) \rVert_2^2$ is a generalized L2 metric that captures how much attribution mass is allocated to the region of interest.
\paragraph{Randomization-based sanity checks}
Randomization-based sanity checks have been developed to verify that an attribution result is not abstract and does indeed depend sufficiently on the model of interest. Typically they asses whether feature attributions change if certain parameters of the model are randomized. If an attribution result is invariant to parameter randomization it might not be reliably explain the examined model. To express this via a generalized $L2$ metric, suppose we are interested in explaining the prediction of a model $f_\theta$ with parameters $\theta$. Let $\phi_{\theta}(x)$ be a feature attribution result obtained from the original model $f_\theta$. Further, let $f_{\tilde{\theta}}$ denote the corresponding model where all parameters or a specific subset is randomized based on $\tilde{\theta} \sim \mathbb{P}_{\tilde{\theta}}$. Then, the variability of $\phi_{\theta}(x)$ under parameter randomization can be computed via   $\mathcal{Q}(\phi_{\theta}(x)) = -\mathbb{E}_{\tilde{\theta}}\lVert \phi_{\theta}(x) -\phi_{\tilde{\theta}}(x) \rVert_2^2$. Note that we incorporate a negative sign to indicate that invariant attribution results correspond to lower quality.

\paragraph{Complexity}
To express complexity measures for feature attributions one can use the truncated $L2$ norm $\lVert \cdot \rVert_{2, t}$ as a sparsity measure \cite{dicker2014sparsity}. This implies that  $\mathcal{Q}(\phi(x)) = \lVert \min\{\phi(x), t\} \rVert_2^2$ where $\min\{ \cdot, \cdot\}$ denotes the elementwise minimum operator and $t$ a predefined noise threshold. Note that improving this metric $\mathcal{Q}$ requires pushing more entries of $\phi(x)$ below the threshold $t$ which also promotes sparsity and reduces complexity. To translate this metric to the generic formulation proposed in Definition 4.1. one needs to set $\gamma_1 \in \mathbb{R}^{d \times d}$ and $\gamma_2 \in \mathbb{R}^{d}$ like this:
\begin{align*}
(\gamma_1)_{i,j}=\begin{cases} &i=j: 1 \; \text{if}  \; \lvert \phi_i \rvert<t \; \text{else} \; 0 \\
&i \ne j: 0
\end{cases} \quad \quad \quad 
(\gamma_2)_{i}=\begin{cases} -t&  \; \text{if}  \; \lvert \phi_i  \rvert >t \\
0& \; \text{else}
\end{cases} 
\end{align*}
\section{Experimental Details}
\label{app:C}
\subsection{Metric Details}
\paragraph{Robustness}
Throughout all experiments Average-Sensitivity $(\text{SENS}_{\textit{AVG}})$ and and Max-Sensitivity $(\text{SENS}_{\textit{MAX}})$ are computed using uniformly distributed corruptions $\varepsilon \sim \mathcal{U}[-0.1, 0.1]$:
\begin{align*}
     \text{SENS}_{\textit{AVG}} : \mathbb{E}_{\varepsilon} \left[ \lVert \phi(x) - \phi(x+\varepsilon) \rVert \right]_2^2 \qquad \text{and} \qquad
     \text{SENS}_{\textit{MAX}} : \max_{\varepsilon} \lVert \phi(x) - \phi(x+\varepsilon) \rVert_2^2  
\end{align*}
To optimize the aggregation weights for $\text{AGG}_{\textit{robust}}$ and $\text{AGG}_{\textit{opt}}$ the expectation is estimated using only $m_{\textit{agg}}$ samples for $\varepsilon$. During the evaluation in section 5.1, the metrics are computed using $m_{\textit{eval}}=200$ unseen samples to explicitly check for generalization. During evaluation both metrics are computed using the implementation provided by \texttt{Quantus} \cite{hedstrom2023quantus}.
\paragraph{Faithfulness}
To compute the Infidelity metric $(\text{INFD})$ we rely on original design principles proposed by the authors \cite{yeh2019fidelity}. In particular, we utilized binary perturbations $I \in \{0,1\}^d$ that randomly select an image area of $20\%$ such that $I^T\phi$ equals the sum of attribution scores allocated to the selected region. From this quantity, we subtract the prediction change caused by replacing the selected image area with the corresponding values of a blurred image version $x_b$. This can be formalized using a map $h: \mathbb{R}^d \times \mathbb{R}^d \times \mathbb{R}^d \rightarrow \mathbb{R}^d$ with  $h(x, x_b, I)_i = (x_b)_i \; \text{if} \; I_i=1$ and $h(x, x_b, I)_i = x_i$ else. We also incorporated the normalization utilized by the authors in their implementation \cite{yeh2019fidelity}.
For the Faithfulness correlation metric $(\text{FCOR})$ we use the same kind of perturbation and use the Pearson Correlation as correlation measure $\text{corr}$ as proposed in \cite{bhatt2021evaluating}. This results in:
\begin{align*}
    \text{INFD}:\mathbb{E}_{I} \left[ ( I^T\phi(x) - (f(x)-f(h(x, x_b, I)))^2 \right] \qquad \text{and} \qquad   \text{FCOR}: \text{corr}_I \left(I^T\phi(x),  f(x)-f(h(x, x_b, I)) \right)
\end{align*} 
To optimize the aggregation weights for $\text{AGG}_{\textit{faith}}$ and $\text{AGG}_{\textit{opt}}$ the expectation and correlation is estimated using only $m_{\textit{agg}}$ samples of $I$. During the evaluation in section 5.1, the metrics are again computed using $m_{\textit{eval}}=200$ fresh samples to explicitly check for generalization. 
\paragraph{Stability}
All stability metrics have been computed based on their implementation in $\texttt{OpenXAI}$ \cite{agarwal2022openxai}. For Relative Representation Stability (RRS) we used the activation of the final layer before the classification happens as underlying representation to compute the metric.
\paragraph{ROAD} To compute the Remove and Debias metric \cite{rong2022consistent} we leveraged the implementation provided by the $\texttt{pytorch-gradcam}$ library \cite{jacobgilpytorchcam}. Therefore, $\text{MoRF}_p$ (Most relevant first) corresponds to the average decrease in confidence for the correct class of an image resulting from removing the $p$ percent of the most important pixels as indicated by an attribution result. Note, that feature removal is performed using noisy linear imputation which has been demonstrated to produce consistent results matching the outcomes of retaining-based metrics such as Remove and Retrain \cite{hooker2019benchmark}. 
\subsection{Feature Attribution Methods and Aggregations}
\paragraph{Individual Methods}
During the experiments, we evaluated in total twelve different feature attribution techniques. The methods Saliency \cite{simonyan2013deep}, InputxGrad \cite{shrikumar2016not}, Guided Backpropagation\cite{springenberg2014striving}, DeepLift \cite{shrikumar2017learning}, Integrated Gradients \cite{sundararajan2017axiomatic}, GradSHAP \cite{lundberg2017unified},  SmoothGrad \cite{smilkov2017smoothgrad}, VarGrad  \cite{adebayo2018sanity}, Shapley Values \cite{castro2009polynomial}, LIME \cite{ribeiro2016should} and Feature Ablation are computed using the corresponding implementation provided by \texttt{Captum} \cite{kokhlikyan2020captum}. For the methods GradCAM \cite{selvaraju2017grad}, EigenCAM \cite{muhammad2020eigen} and GradCAM++ \cite{chattopadhay2018grad} we utilized the \texttt{pytorch-gradcam} \cite{jacobgilpytorchcam} library.
All attribution results are normalized to lie within the range $[0,1]$ by taking the absolute value and rescaling them based on the maximum to ensure comparability. \\
Note that we excluded Guided Backpropagation when evaluating the MLPMixer architecture since the method was originally designed only for networks with the ReLU activation function.\\
For the experiments in section 5.3 we utilized Lime with three different LASSO regularization parameters $\lambda$. More precisely high sparsity regularization corresponds to $\lambda_{\textit{high}}=0.1$, medium to $\lambda_{\textit{medium}}=0.01$ and no regularization uses an ordinary least square regression approach to estimate the Lime coefficients. For the SLIC variant we provided a feature mask using the SLIC algorithm \cite{achanta2012slic} partitioning an image into approximately $100$ superpixels.
\paragraph{Optimized Aggregation} The aggregation weights for our combination approaches are optimized by estimating the underlying $L2$ metric using $m$ metric evaluation samples yielding $\widehat{\text{SENS}}_{\textit{AVG}}$ and $\widehat{\text{INFD}}$ . In particular, we have:
\begin{align*}
    \widehat{\text{SENS}}_{\textit{AVG}}(\phi^{\omega}) &= \frac{1}{m}\sum_{j=1}^{m} \lVert \phi^{\omega}(x) - \phi^{\omega}(x +\varepsilon^{(j)})\rVert_2^2 \\
    \widehat{\text{INFD}}(\phi^{\omega}) &= \frac{1}{m}\sum_{j=1}^{m} \lVert \big(I^{(j)}\big)^T\phi^{\omega}(x) - \big(f(x)-f(h(x, x_b, I^{(j)})\big)\rVert_2^2 
\end{align*} and the weights are computed by solving:
\begin{align*}
    \text{AGG}_{\textit{robust}}:& \quad  
    \omega^{\textit{robust}} = \arg\min_{\omega \in \Omega}\;\widehat{\text{SENS}}_{\textit{AVG}}(\phi^{\omega})\\ 
 \text{AGG}_{\textit{faith}}:& \quad \omega^{\textit{faith}} = \arg\min_{\omega \in \Omega}\;\widehat{\text{INFD}}(\phi^{\omega})  \\
\text{AGG}_{\textit{opt}}:& \quad \omega^{\textit{opt}} = \arg\min_{\omega \in \Omega}\;\widehat{\text{INFD}}(\phi^{\omega}) + \widehat{\text{SENS}}_{\textit{AVG}}(\phi^{\omega})
\end{align*}
All objectives are reformulated as constrained quadratic programs using the logic described in section 4 of the main paper and optimized using the default solver provided by \texttt{cvxpy} \cite{diamond2016cvxpy}. For $\text{AGG}_{\textit{opt}}$ we additionally normalized both metrics using the Frobenius norm of the respective parameter matrix $\lVert \Gamma^T\Gamma \rVert_F$ to ensure comparability between the two considered metrics. 
\subsection{Model Details}
We downloaded all convolutional models,including VGG16 \cite{simonyan2013deep}, AlexNet \cite{krizhevsky2012imagenet}, ResNet18 \cite{he2016deep}, MobileNetV2 \cite{sandler2018mobilenetv2} and DenseNet121 \cite{huang2017densely},  from \texttt{torchvison} with pre-trained weights. All transformer-based models are downloaded using the \texttt{timm} library \cite{rw2019timm}. More precisely, we utilized the following model variants:\\
DeiT \cite{touvron2021training}:\hspace{1cm}\texttt{deit\_tiny\_patch16\_224.fb\_in1k}\\
ViT \cite{dosovitskiy2020image}:\hspace{1cm}\texttt{vit\_tiny\_patch16\_224.augreg\_in21k\_ft\_in1k}\\ SwinT \cite{liu2021swin}:\hspace{2cm}\texttt{swin\_tiny\_patch4\_window7\_224.ms\_in1k}\\
MLPMixer \cite{tolstikhin2021mlp}:\hspace{1cm} \texttt{mixer\_b16\_224.goog\_in21k\_ft\_in1k}

\section{Additional Experiments and Results}
\subsection{Extended results for ROAD}
\label{app:D1}
\begin{table}[h!]
\centering
\caption{Remove and Debiase (ROAD) metric results on a Resnet18 where $\text{MoRF}_{p}$ evaluates the average decrease in confidence caused by removing the top $p$ percent of the most relevant pixels as indicated by the explanation method.}
\resizebox{\textwidth}{!}{
\begin{tabular}{c|ccccccccc|c}
\toprule
Method & $\text{MoRF}_{10} $& $\text{MoRF}_{20} $& $\text{MoRF}_{30} $ & $\text{MoRF}_{40} $& $\text{MoRF}_{50} $ & $\text{MoRF}_{60} $& $\text{MoRF}_{70} $& $\text{MoRF}_{80} $& $\text{MoRF}_{90} $& \textbf{Average} $\downarrow$ \\
\midrule
Deeplift & -1.10 & -2.02 & -3.02 & -4.15 & -5.42 & -6.85 & -8.56 & -10.57 & -13.03 & -6.30 \\
VarGrad & -2.37 & -4.69 & -6.79 & -8.45 & \underline{-10.05} & \underline{-11.35} & \underline{-12.57} & \underline{-13.63} & \underline{-14.67} & -9.84 \\
GuidedBP & -3.09 & -4.86 & -6.31 & -7.62 & -8.94 & -10.17 & -11.35 & -12.59 & -14.01 & -9.77 \\
IntGrad & -0.85 & -1.79 & -2.82 & -3.96 & -5.31 & -6.86 & -8.67 & -10.70 & -13.13 & -6.23 \\
SmoothGrad & -1.57 & -2.72 & -3.84 & -4.94 & -6.09 & -7.35 & -8.74 & -10.38 & -12.73 & -6.82 \\
InputxGrad & -0.63 & -1.34 & -2.24 & -3.37 & -4.63 & -6.18 & -8.05 & -10.23 & -12.87 & -5.72 \\
Saliency & -0.58 & -1.24 & -2.04 & -3.10 & -4.25 & -5.65 & -7.31 & -9.33 & -12.13 & -5.18 \\
\midrule
$\text{AGG}_{\textit{Mean}}$ & -2.21 & -3.98 & -5.68 & -7.33 & -8.82 & -10.27 & -11.60 & -12.92 & -14.31 & -9.57 \\
$\text{AGG}_{\textit{Var}}$ & -2.21 & -3.98 & -5.68 & -7.33 & -8.82 & -10.28 & -11.60 & -12.93 & -14.31 & -9.57 \\
$\text{AGG}_{\textit{faith}}$ & -2.66 & -4.50 & -6.12 & -7.49 & -8.79 & -9.99 & -11.28 & -12.57 & -14.05 & -9.50 \\
$\text{AGG}_{\textit{opt}}$ & \underline{-3.30} & \underline{-5.41} & \underline{-7.14} & \underline{-8.62} & -9.96 & -11.25 & -12.45 & -13.51 & -14.58 & \underline{-10.25} \\
$\text{AGG}_{\textit{robust}}$ & \textbf{-3.36} & \textbf{-5.48} & \textbf{-7.17} & \textbf{-8.78} & \textbf{-10.15} & \textbf{-11.40} & \textbf{-12.65} & \textbf{-13.70} & \textbf{-14.68} & \textbf{-10.37} \\
\bottomrule
\end{tabular}
}
\end{table}
\subsection{Results on other Datasets}
\label{app:D2}
To substantiate the findings in the main paper, we repeated the experiments in section 5.1 on four additional datasets, namely CIFAR10 as well as three medical image datasets BloodMNIST, DermaMNIST and PathMNIST \cite{yang2023medmnist}. Tables 5 and 6 summarize the corresponding results for the considered faithfulness and robustness metrics based on 500 images evaluated with a pre-trained ResNet18 model. 
\begin{table*}[h] \centering
        \caption{$\text{INFD}$ and $\text{FCOR}$ results for different attribution methods and aggregation strategies for a ResNet18 model. Our approach $\text{AGG}_{\textit{faith}}$ consistently outperforms all other techniques and $\text{AGG}_{\textit{opt}}$ is either second best or comparable.}
         \begin{tabular}{c|*{2}{c}|*{2}{c}|*{2}{c}|*{2}{c}}
         \toprule
         Feature
         &  \multicolumn{2}{c}{CIFAR10} 
         &  \multicolumn{2}{c}{BloodMNIST}
         & \multicolumn{2}{c}{DermaMNIST} 
        & \multicolumn{2}{c}{PathMNIST}
        \\ Attribution
       &  $\text{INFD}\downarrow$ & $\text{FCOR}\uparrow$ 
       &   $\text{INFD}\downarrow$& $\text{FCOR}\uparrow$ &   $\text{INFD}\downarrow$& $\text{FCOR}\uparrow$  &
       $\text{INFD}\downarrow$& $\text{FCOR}\uparrow$ \\   
         \midrule
         Saliency &4.129 &0.159  &\underline{11.60}& 0.393&0.324 & 0.321 &5.394  & 0.152\\
         DeepLift &3.928 & 0.252  &15.12 &0.278 &0.354 &0.233 & 5.294 & 0.159\\
         IntGrad &4.016 &0.229  &12.97& 0.290&\underline{0.321} &0.309  & 5.145 &\underline{0.190}\\
         InputxGrad &4.326 &0.133  &13.56&0.206 &0.322 &0.320 &5.466  & 0.118\\ 
         SmoothGrad  &3.736 &0.313  &13.04&0.326 &0.360 &0.276 & 5.395 & 0.127\\
         VarGrad  &3.607 &0.319&11.84  &0.379 &0.380 &0.145 &  6.106& 0.103\\
        \midrule
         $\text{AGG}_{\textit{Mean}}$  &3.802 &0.290& 12.71&0.390 &0.332 & 0.362& 5.153 & 0.183\\
         $\text{AGG}_{\textit{Var}}$  &3.817 &0.290&12.73&0.391 &0.335 &0.361& \underline{5.132} & 0.186\\
         $\text{AGG}_{\textit{opt}}$  \textbf{(ours)} &\underline{3.538} &\underline{0.343} &11.80 &\underline{0.414} & 0.322&\underline{0.370}& 5.135 &0.187\\
         $\text{AGG}_{\textit{faith}}$  \textbf{(ours)} &\textbf{3.342} &\textbf{0.378}&\textbf{10.66}&\textbf{0.465}&\textbf{0.281} &\textbf{0.433}& \textbf{4.850} & \textbf{0.286}\\
         \bottomrule
     \end{tabular}

     \end{table*}

\begin{table*}[h] \centering

     \caption{$\text{SENS}_{\text{AVG}}$ ($\text{S}_{\text{AVG}}$) and $\text{SENS}_{\text{MAX}}$ ($\text{S}_{\text{MAX}}$) results for gradient-based attribution methods and different aggregation strategies for a ResNet18 model. Our approach $\text{AGG}_{\textit{robust}}$ consistently outperforms all other techniques followed by $\text{AGG}_{\textit{opt}}$ as second best.}
         \begin{tabular}{c|*{2}{c}|*{2}{c}|*{2}{c}|*{2}{c}}
         \toprule
         Feature
         &  \multicolumn{2}{c}{CIFAR10} 
         &  \multicolumn{2}{c}{BloodMNIST}
         & \multicolumn{2}{c}{DermaMNIST} 
        & \multicolumn{2}{c}{PathMNIST}
        \\ Attribution
       &  $\text{S}_{\text{AVG}}\downarrow$ & 
       $\text{S}_{\text{MAX}}\downarrow$ & 
       $\text{S}_{\text{AVG}}\downarrow$&
       $\text{S}_{\text{MAX}}\downarrow$ &
       $\text{S}_{\text{AVG}}\downarrow$&
       $\text{S}_{\text{MAX}}\downarrow$&
       $\text{S}_{\text{AVG}}\downarrow$&
       $\text{S}_{\text{MAX}}\downarrow$ \\     
         \midrule
         Saliency &0.916 &1.143&0.696& 0.882&0.787 & 0.993 &0.942 &1.108\\
         DeepLift &0.805 & 1.016&0.514 &0.644 &0.679 &0.936& 0.723&0.818 \\
         IntGrad &0.820 &1.029&0.481& 0.622&0.673 &0.861 & 0.833&0.955 \\
         InputxGrad &0.910 &1.152&0.708&0.893 &0.795 &1.000 &0.932 & 1.076\\ 
         SmoothGrad  &0.818 &0.978&0.549&0.683 &0.566 &0.694 &0.828 &1.000\\
         VarGrad  &0.617 &0.953&0.449&0.677 &\underline{0.390} &0.606 &0.583 &0.961\\
        \midrule
         $\text{AGG}_{\textit{Mean}}$  &0.553 &0.699& 0.384&0.526 &0.475 & 0.615 & 0.537&0.655\\
         $\text{AGG}_{\textit{Var}}$  &0.549 &0.689&0.386&0.526 &0.474 &0.609 &0.536 &0.653\\
         $\text{AGG}_{\textit{opt}}$  \textbf{(ours)} &\underline{0.492} &\underline{0.657} &\underline{0.343} &\underline{0.485} &0.411&\underline{0.567} &\underline{0.457} & \textbf{0.619}\\
         $\text{AGG}_{\textit{robust}}$  \textbf{(ours)} &\textbf{0.491} &\textbf{0.650}&\textbf{0.339}&\textbf{0.476}&\textbf{0.389} &\textbf{0.547} &\textbf{0.439} & \underline{0.627}\\
         \bottomrule
     \end{tabular}

     \end{table*}

\subsection{Computation times of different aggregation strategies}
\label{app:D3}
In Table 7 we report the time required to retrieve optimal aggregation weights across 7 explainers for different models evaluated on an NVIDIA RTX A5000 GPU and averaged over 100 samples with corresponding standard deviations:
\begin{table}[h!]
    \centering
        \caption{Inference times to perform weight optimization based on seven explanation methods for $\text{AGG}_{\textit{faith}}$ and $\text{AGG}_{\textit{robust}}$ as average over 100 samples with corresponding standard deviation. }
    \begin{tabular}{lcccccc}
        \toprule
        Time (s)  & VGG16 & ResNet18 &MobileNetV2 & DenseNet121 & DeiT & SwinT \\
        \midrule
        $\text{AGG}_{\textit{faith}}$  &0.79 $\pm0.06$  &0.75 $\pm 0.06$  &0.77 $\pm 0.06$&   1.49 $\pm 0.51$&1.33 $\pm0.92$  &0.83 $\pm 0.09$ \\
        $\text{AGG}_{\textit{robust}}$ & 22.57 $\pm0.13$ &7.54 $\pm 0.25$  &13.06 $\pm0.35$&37.22 $\pm 4.01$& 17.96 $\pm1.78$  &33.34 $\pm 1.17$ \\
        $\text{AGG}_{\textit{opt}}$ & 23.39 $\pm 0.14$ &8.31 $\pm0.23$  & 13.91 $\pm0.27$&38.52 $\pm 3.28$&19.39 $\pm4.59$  &34.47 $\pm 1.13$ \\
        \bottomrule
    \end{tabular}

\end{table}

We believe the additional computational cost imposed by our aggregation technique is minor compared to the strong improvements in explanation metrics.

\subsection{Ablation Studies regarding number and diversity of combined explanations}  
\label{app:D4}
\paragraph{Varying the number of methods to be aggregated}
We anticipate that our method will benefit from an increasing number of considered attributions by automatically down-weighting disadvantageous explanations. This behaviour is also exemplified in the last two rows of Figure 2 where deteriorated results received zero weight. To further investigate this, we performed a dedicated experiment in which we increased the number of feature attribution methods incrementally from 2 to 7 on a ResNet18 over 100 samples. The results in Table 8 and 9 show that the metrics do indeed get better for robustness and faithfulness, but the improvements seem to saturate at a certain point. The ordered set of explainers that were used for this experiment is: DeepLift, VarGrad, GuidedBackprop, SmoothGrad, IntGrad, InputxGrad, Saliency. 
\begin{table}[h!]
\centering
\begin{minipage}{.45\linewidth}
    \centering
       \caption{Robustness metrics for $\text{AGG}_{\textit{robust}}$ combining an increasing number of explanation methods}
    \begin{tabular}{lcccccc}
        \toprule
        $\text{AGG}_{\textit{robust}}$ & 2 & 3 & 4 & 5 & 6 & 7 \\
        \midrule
        $\text{S}_{\text{AVG}} \downarrow$  & 0.52 & 0.43 & 0.42 & 0.41 & 0.41 & 0.41 \\
        $\text{S}_{\text{MAX}} \downarrow$  & 0.68 & 0.54 & 0.52 & 0.52 & 0.51 & 0.50 \\
        \bottomrule
    \end{tabular}

\end{minipage}%
\hspace{1cm}
\begin{minipage}{.45\linewidth}
    \centering
        \caption{Faithfulness metrics for $\text{AGG}_{\textit{faith}}$ combining an increasing number explanation methods}
    \begin{tabular}{lcccccc}
        \toprule
        $\text{AGG}_{\textit{faith}}$ & 2 & 3 & 4 & 5 & 6 & 7 \\
        \midrule
        INFD $\downarrow$ & 2.69 & 2.46 & 2.44 & 2.43 & 2.43 & 2.43 \\
        FCOR $\uparrow$ & 0.46 & 0.49 & 0.50 & 0.51 & 0.51 & 0.50 \\
        \bottomrule
    \end{tabular}

\end{minipage}
\end{table}

\paragraph{Combining different types of explanation techniques}
Concerning the diversity of explainers to use, we argue that our approach can be applied to fruitfully combine all methods that output fairly comparable explanations. This includes gradient-based and perturbation-based ones and we expect that also a higher diversity will be advantageous. To further investigate this, we evaluated our method again on 100 samples with a ResNet18 using 3 gradient-based methods, 3 perturbation-based methods, and a combination of all 6. The results in Tables 10 and 11 indicate that including both types does also benefit our aggregation approach. For instance, the perturbation-based methods seem to be significantly more robust than the gradient-based ones but by combining them, we can even further improve their robustness. We can also enhance faithfulness this way. The gradient-based methods used in this experiment are DeepLift, SmoothGrad, InputxGrad and the perturbation-based methods used in this experiment are Lime, Shapley Values and Feature Ablation.\\
All methods have been computed based on their default implementation in \texttt{Captum} \cite{kokhlikyan2020captum} while we additionally used a feature mask of 16x16 patches for all perturbation-based methods.

\begin{table}[h!]
\centering
\begin{minipage}{.45\linewidth}
    \centering
        \caption{Robustness metrics for $\text{AGG}_{\textit{robust}}$ combining different types of explanation methods}
    \begin{tabular}{lccc}
        \toprule
        $\text{AGG}_{\textit{robust}}$ & Grad.-based & Pert.-based & Both \\
        \midrule
        $\text{S}_{\text{AVG}} \downarrow$ & 0.65 & 0.49 & \textbf{0.46} \\
        $\text{S}_{\text{MAX}}\downarrow$ & 0.77 & 0.61 & \textbf{0.57} \\
        \bottomrule
    \end{tabular}

    \label{tab:agg_robust}
\end{minipage}%
\hspace{1cm}
\begin{minipage}{.45\linewidth}
    \centering
        \caption{Faithfulness metrics for $\text{AGG}_{\textit{faith}}$ combining different types of explanation methods}
    \begin{tabular}{lccc}
        \toprule
        $\text{AGG}_{\textit{faith}}$ & Grad.-based & Pert.-based & Both \\
        \midrule
        INFD $\downarrow$ & 2.82 & 2.92 & \textbf{2.76} \\
        FCOR $\uparrow$ & 0.46 & 0.44 & \textbf{0.47} \\
        \bottomrule
    \end{tabular}

\end{minipage}
\end{table}


\end{document}